%% file: AA_main_paper.tex
\documentclass[letter, 11pt]{article}
\usepackage[hyphens]{url}
\usepackage[letterpaper, margin=1in]{geometry}
\usepackage[utf8]{inputenc}
\usepackage[numbers]{natbib}
\usepackage{hyperref}

\input{macros}
\title{
How Global Calibration Strengthens Multiaccuracy
}

\author{S\'ilvia Casacuberta\\University of Oxford \and 
Parikshit Gopalan\\Apple
\and Varun Kanade\\University of Oxford
\and Omer Reingold\\Stanford University}
\date{\today}

\begin{document}

\maketitle

\begin{abstract}
    Multiaccuracy and multicalibration are multigroup fairness notions for prediction that have found numerous applications in learning and computational complexity \cite{hkrr2018}. They can be achieved from a single learning primitive: weak agnostic learning. A line of work starting from \cite{omni} has shown that multicalibration implies a very strong form of learning. Here we investigate the power of multiaccuracy as a learning primitive, both with and without the additional assumption of calibration. We find that multiaccuracy in itself is rather weak, but that the addition of global calibration (this notion is called \emph{calibrated multiaccuracy}) boosts its power substantially, enough to recover implications that were previously known only assuming the stronger notion of multicalibration.
    
    We give evidence that multiaccuracy might not be as powerful as standard weak agnostic learning, by showing that there is no way to post-process a multiaccurate predictor to get a weak learner, even assuming the best hypothesis has correlation $1/2$. Rather, we show that it yields a restricted form of weak agnostic learning, which requires some concept in the class to have correlation greater than $1/2$ with the labels. However, by also requiring the predictor to be calibrated, we recover not just weak, but strong agnostic learning. 

    A similar picture emerges when we consider the derivation of hardcore measures from predictors satisfying multigroup fairness notions \cite{TrevisanTV09, casacuberta2024complexity}. On the one hand, while multiaccuracy only yields hardcore measures of density half the optimal, we show that (a weighted version of) calibrated multiaccuracy achieves optimal density.  

    Our results yield new insights into the complementary roles played by multiaccuracy and calibration in each setting. They shed light on why multiaccuracy and global calibration, although not particularly powerful by themselves, together yield considerably stronger notions.
    
\end{abstract}
\thispagestyle{empty}
\newpage
\tableofcontents
\thispagestyle{empty}
\newpage
\setcounter{page}{1}

\input{1-intro}

\input{overview}

\input{2-prelims}
\input{3-wal-ma}
\input{4-hardcore}

\input{6-meek-wal}
\input{8-conclusions}


\bibliographystyle{alpha}
\bibliography{refs}

\clearpage
\appendix
\input{9-projection}

\end{document}

%% file: macros.tex
\usepackage[utf8]{inputenc}
\usepackage[T1]{fontenc}
\usepackage{amsmath, amssymb, amsthm, bbm, graphicx, url}
\usepackage{amsfonts,fullpage}
\usepackage{float}
\usepackage{graphicx}

\usepackage{bbm}
\usepackage{xcolor}
\usepackage[boxruled,linesnumbered,vlined]{algorithm2e}
\hypersetup{
    colorlinks=true,
    linktocpage=true,
    linkcolor=blue!50!black,  
    citecolor=blue!50!black,  
    urlcolor=blue!50!black    
}
\usepackage{cleveref}
\usepackage{dsfont}
\usepackage{xspace}
\usepackage{comment}
\usepackage{tikz}

\usepackage{enumitem}

\newcommand{\eat}[1]{}

\newtheorem*{definition*}{Definition}
\newtheorem*{proposition*}{Proposition}
\newtheorem*{corollary*}{Corollary}

\newtheorem{theorem}{Theorem}[section]
\newtheorem{lemma}[theorem]{Lemma}
\newtheorem{definition}[theorem]{Definition}

\newtheorem{corollary}[theorem]{Corollary}

\newtheorem{assumption}[theorem]{Assumption}

\newtheorem{remark}[theorem]{Remark}

\newcommand{\wh}{\widehat}

\newcommand{\R}{\mathbb{R}}
\newcommand{\X}{\mathcal{X}}
\newcommand{\MA}{\mathsf{MA}}

\newcommand{\dns}{\mathsf{dns}}

\newcommand{\mA}{\mathcal{A}}

\newcommand{\mC}{\mathcal{C}}
\newcommand{\mD}{\mathcal{D}}

\newcommand{\hmH}{\mu_{\mathsf{Max}}}
\newcommand{\muMax}{\mu_{\mathsf{Max}}}

\newcommand{\wmax}{w_{\mathsf{Max}}}
\newcommand{\ttvH}{\mu_{\mathsf{TTV}}}
\newcommand{\muTTV}{\mu_{\mathsf{TTV}}}

\newcommand{\mF}{\mathcal{F}}

\newcommand{\mX}{\mathcal{X}}

\newcommand{\cor}{\mathrm{cor}}

\newcommand{\tp}{\tilde{p}}

\newcommand{\w}{\mathbf{w}}
\newcommand{\x}{\mathbf{x}}
\newcommand{\z}{\mathbf{z}}

\newcommand{\y}{\mathbf{y}}

\newcommand{\lt}{\left}
\newcommand{\rt}{\right}
\newcommand{\ecc}{\mathrm{Ecc}}
\newcommand{\rsh}{\mathrm{RSH}}
\newcommand{\rsht}{\mathrm{RSH}}

\newcommand{\zo}{\ensuremath{\{0,1\}}}

\renewcommand{\hat}{\wh}
\newcommand{\eps}{\varepsilon}

\newcommand{\abs}[1]{\ensuremath \Bigl\lvert #1 \Bigr\rvert}

\newcommand{\la}{\ensuremath{\langle}}
\newcommand{\ra}{\ensuremath{\rangle}}

\newcommand{\GR}[1]{}
\newcommand{\PG}[1]{}
\newcommand{\LH}[1]{}

\newcommand{\pmo}{\ensuremath{ \{\pm 1\} }}
\newcommand{\fr}[1]{\ensuremath{\frac{1}{#1}}}

\newcommand{\err}{\mathrm{err}}
\newcommand{\opt}{\mathrm{Opt}}

\newcommand{\ind}[1]{\ensuremath{\mathbf{1}[#1]}}

\newcommand{\C}{\mathcal{C}}

\DeclareMathOperator{\poly}{poly}

\DeclareMathOperator*{\E}{\mathbf{E}}

\DeclareMathOperator{\clip}{clip}

\newcommand{\barmu}{\bar{\mu}}

\newcommand{\mpk}[1]{}

\newcommand{\sign}{\mathrm{sign}}

\newcommand{\maj}{\mathrm{MAJ}}

\newcommand{\ECE}{\mathsf{ECE}}

\newcommand{\EAE}{\mathsf{EAE}}

\newcommand{\lspan}{\mathsf{span}}

\newif\iforange
\orangefalse

%% file: 1-intro.tex
\section{Introduction}

A major line of research within the algorithmic fairness community focuses on the notion of {\em multicalibration,} first introduced by H\'ebert-Johnson, Kim, Reingold, and Rothblum \cite{hkrr2018}. This concept centers on risk predictors that estimate individual probabilities for outcomes such as: what is the likelihood that an individual will repay a loan, successfully complete their studies, or be hired after an interview? These predictive tasks often inform high-stakes decisions—granting loans, admitting applicants to universities, or extending job interview invitations—and thus raise critical fairness concerns.
Multicalibration requires that predictions be {\em calibrated} across a rich and potentially intersecting collection of subpopulations. A predictor is said to be calibrated on a subpopulation if, for individuals to whom it assigns a probability $v$ of a given outcome, approximately a $v$ fraction of them actually experience that outcome. In other words, the predictions should ``mean what they say'' not just globally, but also within each relevant subpopulation.

While multicalibration gained widespread acceptance as an appealing fairness criterion, this paper focuses on its interaction with concepts from other domains, specifically complexity theory and computational learning theory. Multicalibration draws from both fields. It is naturally viewed as a complexity-theoretic perspective on fairness, as it asks for predictors to be able to pass a rich class of computationally-defined tests. This connection has been sharpened through the lens of outcome indistinguishability \cite{OI}. From the standpoint of learning theory, multicalibration is only meaningful as a fairness notion if it is computationally attainable. The algorithms used to learn multicalibrated predictors rely on foundational concepts from learning theory, such as weak agnostic learning and agnostic boosting~\cite{Ben-DavidLM01, MansourM2002, KalaiMV08, kk09, fel09}. 

A large body of work has shown deep connections in the other direction; namely, exploring the impact of multicalibration and its relaxations to complexity theory and learning. Most relevant to this paper, multicalibration implies a strong form of loss minimization called \emph{omniprediction} \cite{omni, lossOI, gopalan2024swap, okoroafor2025near} and constructions of hardcore distributions in computational complexity \cite{casacuberta2024complexity}. Some of the other implications are the notion of domain adaptation \cite{kim2022universal}, and additional complexity-theoretic applications \cite{DworkLLT, hu2023generative, casacuberta2024complexity, marcussen2024characterizing}. 
In the application of multicalibration to computational complexity and learning theory, it is natural to ask whether multicalibration is truly necessary and what the minimal underlying assumptions are. Beyond potential efficiency improvements, our goal is to offer new insights into the fundamental concepts of agnostic learning and hardcore distributions.

\paragraph{Calibrated Multiaccuracy.}
Multiaccuracy \cite{TrevisanTV09, hkrr2018, kgz} is a weaker notion of multigroup fairness. Like multicalibration, it is parameterized by a family of subpopulations $\mC$. For the purposes of this discussion, we can assume that each $c \in \mC$ is a Boolean function, that is, the characteristic function of a set.\footnote{The formal definition allows for real-valued functions which can be interpreted as encoding fuzzy set membership.} In multiaccuracy, we require that for each $c \in \mC$, the expected prediction over individuals in the set defined by $c$ is approximately equal to the fraction of those individuals who experience a positive outcome. Hence, such predictors are accurate on average for each such $c$.
In this work, we explore the power of multiaccuracy as a primitive in computational learning and in the context of hardcore distributions. In both settings, a common theme arises: multiaccuracy alone is insufficient to achieve our desired guarantees. 

However, when we supplement multiaccuracy (or a weighted variant thereof) with \emph{global} calibration—that is, calibration over the entire population rather than all subgroups in $\mC$—we recover many of the appealing consequences of multicalibration.
This strengthened notion, referred to as \emph{calibrated multiaccuracy}, has already demonstrated its utility in prior work on omnipredictors via outcome indistinguishability \cite{lossOI}. In \cite{lossOI}, it is also argued that calibrated multiaccuracy can be more efficient to achieve than full multicalibration, and costs about the same as multiaccuracy. For this reason, our results give efficiency improvements compared with basing the applications we consider on multicalibration instead. Furthermore, by clarifying the distinct but complementary contributions of multiaccuracy and global calibration in each application, this work offers new insights on these fundamental notions.

\paragraph{Computational learning.} 
Multicalibrated predictors with respect to a family $\mC$ can be learned from a sample of individuals and outcome pairs via a learning primitive known as \emph{weak agnostic learning} for $\mC$  \cite{hkrr2018}. Weak agnostic learning allows learning a hypothesis that is somewhat correlated with the labels as long as some hypothesis in $\mC$ is sufficiently correlated \cite{Ben-DavidLM01, KalaiMV08, kk09, fel09}. Quantitatively, assuming the existence of a hypothesis in $\mC$ that achieves correlation $\alpha > 0$, an $(\alpha, \beta)$ weak learner returns a hypothesis with correlation at least $\beta$. Typically, one hopes to get $\beta =\poly(\alpha)$ even when $\alpha$ approaches $0$. This should be contrasted with (strong) agnostic learning, where a hypothesis can be found that is almost as correlated with the data as the best hypothesis in the class $\C$ is; in other words $\beta \geq \alpha -\eps$. 

A line of work \cite{omni, lossOI, gopalan2024swap, okoroafor2025near} shows that multicalibration implies a strong form of learning called \emph{omniprediction} that allows for the simultaneous minimization of all convex, Lipschitz loss functions. All one needs to do is to post-process the predictions with a best-response function. This a very broad guarantee.  Specializing to the case of correlation loss tells us that multicalibration implies a strong agnostic learning algorithm, where the post-processing function here corresponds to $\sign(p - 1/2)$. Thus, agnostic learning is both necessary and sufficient for multicalibration.

Can the same be said about multiaccuracy? Obviously, weak agnostic learning is sufficient for multiaccuracy (which is a weaker notion than multicalibration), and the original multiaccuracy algorithm from \cite{hkrr2018} builds a multiaccurate predictor by repeatedly calling a weak agnostic learner. But, is it necessary? 
While it seems intuitive that the answer should be yes, this paper paints a much more subtle picture: 

\begin{itemize}
 
\item Perhaps surprisingly, we give evidence that multiaccuracy in itself need not imply weak agnostic learning in the standard sense. We prove the existence of a class $\mC$, a distribution on labels, and a $\mC$-multiaccurate predictor $p$ such that even though there is hypothesis in $\C$ that gives correlation $1/2$, no post-processing of the predictor has any non-zero correlation with the labels. Furthermore, we show that this is the case for all classes $\C$ for which there exist of pseudorandom functions that fool $\C$.

\item We complement this lower bound by showing that multiaccuracy does imply weak learning once the correlation $\alpha$ exceeds $1/2$: If there is a hypothesis in $\mC$ with correlation $\alpha$, then $\sign(p -1/2)$ where $p$ is $(\mC, 0)$-multiaccurate achieves correlation $\beta= 2\alpha -1$, a guarantee that is only non-trivial when $\alpha > 1/2$. We call this \emph{restricted weak agnostic learning}, to contrast it with the standard notion of weak agnostic learning where $\alpha$ approaches $0$. 

\item In contrast, we show that if the predictor $p$ is both $\mC$-multiaccurate and calibrated, then $\sign(p - 1/2)$ is a strong agnostic learner for $\mC$, whose correlation is at least that of the best hypothesis in $\mC$. Thus, the addition of calibration considerably increases the power of the predictor.

\item Given the agnostic boosting results of~\cite{kk09, fel09} that show that $(\alpha, \beta)$-weak agnostic learning can be boosted to $(\gamma, \gamma - \alpha)$-weak learning for any $\gamma > \alpha$, one might wonder, if boosting in the opposite direction is possible, i.e., given an $(\alpha, \beta)$-weak learner, can we design an $(\alpha^\prime, \beta^\prime)$-weak learner for $\alpha^\prime < \alpha$? We answer this question in the negative under standard cryptographic assumptions.

\item Lastly, the result of~\cite{hkrr2018} shows that $(\alpha, \beta)$-weak agnostic learning achieves a multiaccuracy parameter of $\alpha$; we show that in general, given such a weak learner, it is not possible to achieve multiaccuracy parameter better than $\alpha/2$. 
\end{itemize}

\paragraph{Hardcore set constructions.} Impagliazzo's hardcore lemma is a fundamental result in computational complexity that relates a function $g$ being (mildly) $\delta$-hard to predict under a distribution $\mD$ for distinguishers in a class $\mC_q$
to the existence of a hardcore measure $\mu$ of density $\Omega(\delta)$ in $\mD$, such that $g$ is $1/2 -\eps$-hard to predict 
(i.e., one cannot improve over random guessing) when sampling from the measure to distinguishers from a family $\mC \subset \mC_q$. The parameter $q$ measures the number of $\mC$-gates required for functions in $\mC_q$. The larger $q$ is, the stronger the hardness assumption we make to show the existence of a hardcore distribution for $\mC$. 

The work of Trevisan, Tulsiani, and Vadhan~\cite{TrevisanTV09}, which preceded the literature on multigroup fairness, gave an elegant connection to multiaccuracy, showing how to get a simple and explicit hardcore measure from a $\mC$-multiaccurate predictor.\footnote{As pointed out in \cite{DworkLLT, casacuberta2024complexity}, the notion of $(\C, \epsilon)$-indistinguishability in complexity theory, as used in \cite{TrevisanTV09}, is exactly equivalent to the notion of $(\C, \epsilon)$-multiaccuracy as defined in the multigroup fairness literature.} Their construction, however, achieves a hardcore measure of density $\delta$, while the optimal density is $2\delta$ \cite{hol05}. 
Recently, Casacuberta, Dwork, and Vadhan \cite{casacuberta2024complexity} showed that multicalibration implies a stronger form of the hardcore lemma, which they call ICHL$++$.  From this they derive the original hardcore lemma by Impagliazzo, this time with a hardcore measure of optimal density $2\delta$. However, while \cite{TrevisanTV09} achieve $q = O(1/\eps^2\delta^2)$, \cite{casacuberta2024complexity} get $q = O(1/(\eps^{12}\delta^6))$.%
\footnote{The stronger bound  of $O(1/(\eps^2\delta^2))$ stated in \cite{casacuberta2024complexity} appears not to be correct. We refer the reader to a detailed discussion at the end of Section \ref{sec:opt}.}
This increased complexity comes in part because achieving multicalibration requires more calls to a (strong) agnostic learner for $\mC$, in contrast to multiaccuracy.%

We show how to construct hardcore distributions with optimal density $2\delta$ as in \cite{casacuberta2024complexity} but with the lower complexity bound of $q = O(1/(\eps^2\delta^2))$ for $\C_{q}$ as in \cite{TrevisanTV09}. Indeed, we give a simple and explicit form for such a measure, starting from a predictor satisfying calibration and a weighted variant of multiaccuracy that we call \emph{weighted multiaccuracy}. Recall that in multiaccuracy we ask that the expected prediction on every set is accurate. In weighted multiaccuracy we instead consider a weighted average (modified by a fixed function of the predictions where we weigh individuals with prediction value closer to half more).  Both calibration and weighted multiaccuracy are implied by multicalibration, so this gives an alternate proof of the optimal density result of \cite{casacuberta2024complexity}. But crucially, the cost of obtaining both these conditions together (by many measures, including oracle calls to the weak learner) is not much more than that of multiaccuracy, which yields a considerably better bound on the parameter $q$.


\paragraph{The three tiers of multigroup fairness.}

As shown in Figure \ref{fig:tiers}, our work together with the results of \cite{TrevisanTV09, casacuberta2024complexity, gopalan2024swap} paints a clear picture of the relative power of various multigroup fairness notions in the context of learning and hardcore measure constructions. Arrows represent implications that either follow by definition or through post-processing the underlying object. When more powerful oracle reductions are allowed, all these notions are achievable using (weak) agnostic learning.

\begin{figure}[!ht]
\centering
\includegraphics[width=12cm]{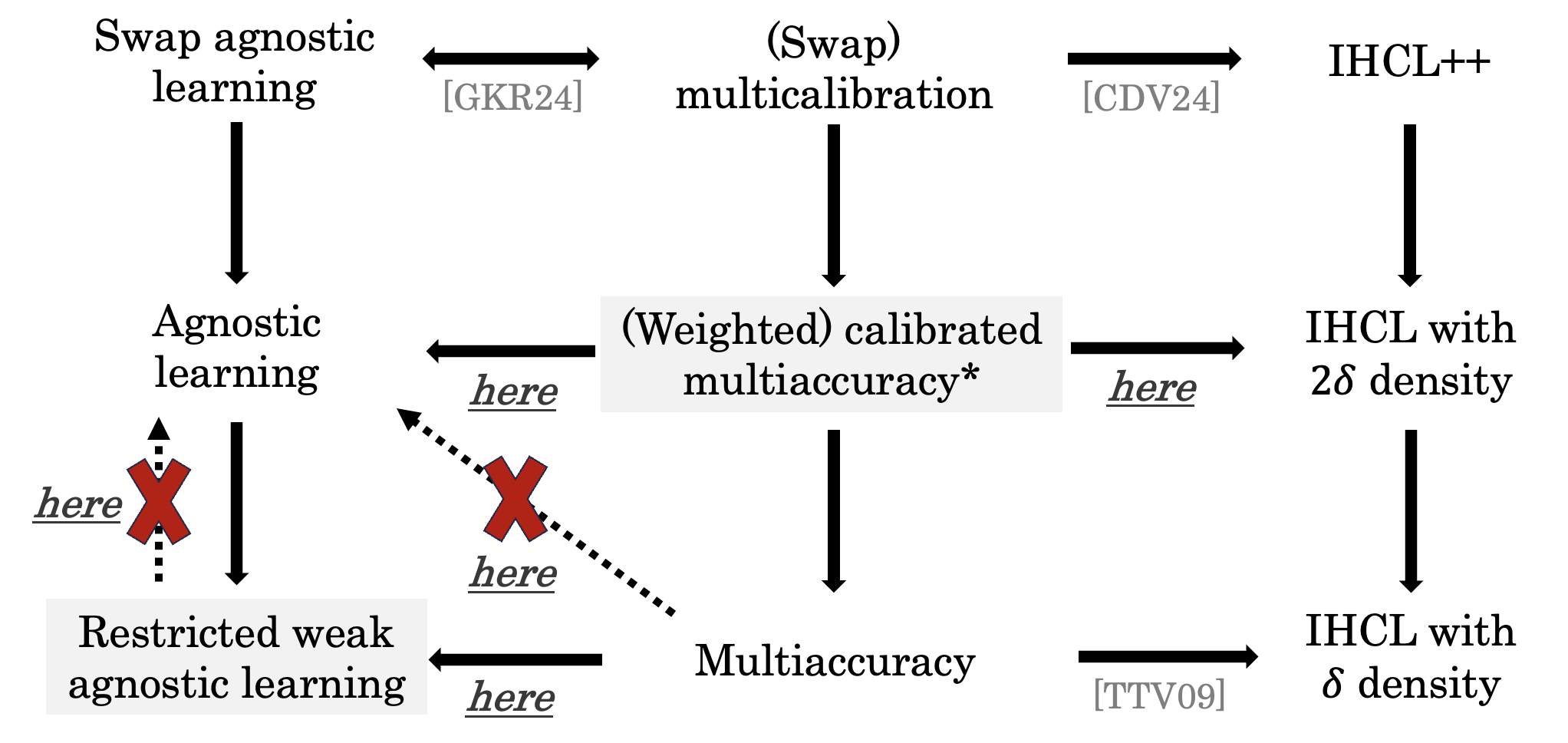}
\caption{Three tiers of multigroup fairness notions.}
\label{fig:tiers}
\end{figure}

The tiers are anchored by multiaccuracy, calibrated multiaccuracy and multicalibration respectively. For simplicity, we ignore the differences between swap multicalibration and standard multicalibration, and weighted multiaccuracy and standard multiaccuracy, since the computational costs of these notions are roughly the same.  

\begin{itemize}
    \item {\bf Multicalibration: the top tier.} Multicalibration implies the strongest known form of the Impagliazzo hardcore lemma, namely the IHCL$++$ \cite{casacuberta2024complexity}. On the learning front, it implies several strong forms of learning: swap agnostic learning and swap omniprediction (we refer the readers to \cite{omni, gopalan2024swap} for the definitions of these notions). It however comes at the cost of higher complexity of learning, more complex predictors (in terms of circuit size) and higher sample complexity.

    \item {\bf Calibrated Multiaccuracy: the Goldilocks zone.} On the learning front, calibrated multiaccuracy implies strong agnostic learning and loss OI for generalized linear models (see \cite{lossOI} for definitions). On the hardcore lemma front, it gives hardcore sets of optimal density $2\delta$, but with relatively low blowup in $q$. The computational, and representational, costs are much lower than multicalibration, and comparable to multiaccuracy. But it is not as powerful as full-blown multicalibration, and does not imply IHCL$++$ or swap agnostic learning (see \cite{gopalan2024swap} for the latter separation). 

    \item {\bf Multiaccuracy: the lowest rung.} Multiaccuracy implies hardcore sets of density $\delta$ by \cite{TrevisanTV09}. On the learning front, it does not give weak agnostic learning for $\alpha \leq 1/2$, using post-processing. It does give restricted weak agnostic learning for $\alpha > 1/2$. The key argument in its favor is that it is cheap and easy to achieve, both in theory and practice: minimizing squared loss or running a GLM over a base class $\mC$ is guaranteed to yield multiaccurate predictors \cite{lossOI}. 
    
\end{itemize}
Independent of our work, Pranay Tankala has observed that getting optimal density $2\delta$ in the IHCL was possible under weaker multigroup fairness conditions than multicalibration \cite{PT}.

\paragraph{The case for calibrated multiaccuracy.}

The power of multicalibration \cite{hkrr2018, OI, omni} illustrates the potency of combining calibration and multiaccuracy. But this great power comes at a high computational price. Our results show that in two important settings, agnostic learning and hardcore set constructions, global calibration together with multiaccuracy is a cheaper alternative to multicalibration, adding weight to a thesis that was first put forth by \cite{lossOI}.  Our proofs highlight the complementary roles that multiaccuracy and calibration play, and give clear intuition for why their combination succeeds even when they each fail in isolation. 
In concurrent work, Hu and Vadhan use the notion of calibrated multiaccuracy to provide pseudoentropy characterizations for general entropy
notions \cite{hu2025generalized}.
A deep insight in theoretical computer science has been that hardness and learning are two sides of the same coin \cite{imp95, ks03, TrevisanTV09, fel09}; our work provides further evidence of this phenomenon.

We believe that these theoretical results also offer some insights to practice. As pointed out in the work of \cite{lossOI}, some form of multiaccuracy is achieved by commonly used machine learning models that perform empirical risk minimization with a proper loss. Such a model will not likely be calibrated as is; see the discussion in \cite{BlasiokGHN23}. But calibrated multiaccuracy is not much harder to achieve, and it can be easily implemented using standard libraries for calibration and regression, in contrast to full-blown multicalibration which is not easily implemented via standard libraries.\footnote{While implementations do exist \cite{PfistererKDSKB21, Globus-HarrisHK23, hansen2024multicalibration}, experiments have found that its sample complexity is fairly large in contrast with multiaccuracy \cite{GopalanRSW22}.} The argument for adding calibration is typically that calibration is a nice property to have, and it does not harm the loss (as long as it is proper). Our work gives a more compelling argument; namely, that global calibration significantly bolsters the power of a predictive model beyond multiaccuracy, without incurring much additional cost.

%% file: overview.tex
\section{Overview of our Results}

In this section, we describe our main results in more detail and explain the key ideas behind them. For this, we quickly recap some notation and definitions, leaving a more formal treatment to Section~\ref{sec:preliminaries}. 

Let $\mD$ denote a distribution over pairs $(\x, \y)$ where $\x \in \X$ denotes a point and $\y \in \zo$ its label. All expectations unless mentioned otherwise refer to this distribution. A predictor is a function $p:\X \to \zo$ while a hypothesis class $\mC =\{c: \X \to [-1,1]\}$ is a collection of bounded functions, which we assume is closed under negation and contains the constant $\mathbf{1}$ function. The Bayes optimal predictor is $p^*(\x) = \E[\y|\x]$. We will think of functions from $\X \rightarrow \R$ as vectors with the inner product $\la p, q\ra =\E[p(\x)q(\x)]$. 

\paragraph{Multigroup fairness notions.} The notion of $(\mC, \tau)$-multiaccuracy as defined in \cite{hkrr2018} requires that
 \[ \max_{c \in \mC}\abs{\E_{\mD}[c(\x)(\y - p(\x))]} \leq \tau. \]
  For intuition, let $\tau = 0$. In this case multiaccuracy requires that for every $c \in \mC$, we have 
  \[ \E[c(\x)(p^*(\x) - p(\x))] = 0,\] 
  or equivalently that every $c \in \mC$ is orthogonal to $p - p^*$. Hence $p$ and $p^*$ both have the same projection onto the subspace spanned by $\mC$. 

  Calibration is a classic notion from the forecasting literature \cite{Dawid}. Perfect calibration asks that $\E[\y|p(\x)] = p(\x)$, in other words, conditioned on any prediction value $v$, the expectation of the label is exactly $v$. We say that the predictor is $\tau$-calibrated if
  \[ \E[|\E[\y|p(\x)] - p(\x))|] \leq \tau.\]
  Calibrated multiaccuracy introduced in \cite{lossOI}, asks for predictors that are both multiaccurate and calibrated, as the name suggests. It is a stronger requirement than just multiaccuracy, although of comparable sample and computational complexity  \cite[Section 7]{lossOI}. 

  Multicalibration \cite{hkrr2018} is a further strengthening of calibrated multiaccuracy. We say that $p$ is $(\mC, \tau)$-\emph{multicalibrated} if
    \[ \max_{c \in \mC}\E[|c(\x)(\E[\y|p(\x)] - p(\x))|] \leq \tau.\]
  If we take $c =\mathrm{1}$, we recover calibration, and if we ignore the absolute values in the outer expectation, we recover multiaccuracy. 

\paragraph{Agnostic learning.} We define weak and strong agnostic learning following \cite{Ben-DavidLM01, kalai2005boosting, KalaiMV08, kk09,  fel09}. The correlation between a function $c:\X \to [-1,1]$ and the labels is defined as $\cor_\mD(\y, c) = \E[c(\x)(2\y -1)] \in [-1,1]$, where the map $\y \mapsto 2\y -1$ maps labels in $\zo$ to $\pmo$. A $\beta$-weak learner is a function $h:\X \to [-1,1]$ such that $\cor(\y, h) \geq \beta$. Taking $h = 0$ which gives $0$ correlation as our baseline,  we would like $\beta \in (0,1]$ to be as large as possible. For a predictor $p:\X \to \zo$, we will often use the transformation $p \mapsto 2p -1$. This maps the $p =1/2$ predictor to $h = 0$, corresponding to the fact that guessing randomly at every point gives $0$ correlation. 

For $\alpha \geq \beta \in [0,1]$, an $(\alpha, \beta)$-weak agnostic learner for $\mC$  is an algorithm that has access to a distribution $\mD$ (with a known marginal $\mD_\X$ over $\X$). If $\mD$ is such that if there exists a hypothesis $c$ such that $\cor(\y, c) \geq \alpha$, the algorithm is guaranteed to return a $\beta$-weak learner. It helps to think of $\beta$ as smaller than but polynomially related to $\alpha$.  In contrast, a strong agnostic learner must return a hypothesis that achieves correlation $\opt(\mC) - \eps$ where $\opt(\mC) = \max_{c \in \mC}\cor(\y, c)$ for every $\eps > 0$. The agnostic boosting results of \cite{kk09, fel09} imply that an $(\alpha, \beta)$ weak learner for $\mC$ can be boosted to a hypothesis that gives correlation at least $\opt(\mC) - \alpha - \eps$. Thus, given a weak learner that can handle arbitrarily small correlation $\alpha$, we can achieve strong agnostic learning.\footnote{The parameter $\beta$ only affects the number of iterations of boosting, hence it is less important than $\alpha$.}

\paragraph{Multigroup fairness and agnostic learning.}
Results of \cite{hkrr2018, lossOI} show that all of multiaccuracy, calibrated multiaccuracy, and multicalibration with error $\tau$ can be achieved using an $(\tau, \beta)$ weak agnostic learner as an oracle. The algorithms will involve multiple calls to the weak agnostic learner. The number of calls grows as $O(1/\beta^2)$ for multiaccuracy and calibrated multiaccuracy, and $O(1/\tau^2\beta^4)$ for multicalibration. The added complexity in multicalibration comes from having to ensure multiaccuracy holds for prediction values that have fairly small probability, see the detailed discussion in \cite[Section 7.3]{lossOI}.

Our goal is to understand the relative strengths of the various multigroup fairness notions for agnostic learning. More precisely, given a $\mC$-multiaccurate predictor $p$, when can we post-process its predictions (by some function $k:[0,1] \to [-1,1]$) to get a weak agnostic learner? Here is what was known before:

\begin{itemize}
	\item When $p$ is \emph{multicalibrated}, then $\sign(p -1/2)$ is a strong agnostic learner. This is an implication of a general result proved in \cite{omni}, that a $\C$-multicalibrated $p$ is an omnipredictor for all convex Lipschitz loss functions. This means that for every such loss $\ell$, there is a post-processing function $k_\ell$ so that $k_\ell(p)$ suffers a loss comparable to the best hypothesis in $\mC$. The implication for agnostic learning comes from considering the (negation of) correlation as a loss.

    \item When $\mC = \{c: \X \to \pmo\}$ is a family of Boolean rather than bounded functions, then correlation coincides with the $0$-$1$ loss. In this setting, the work of \cite{lossOI} shows that if $p$ satisfies calibrated multiaccuracy, then it is an omnipredictor,  using the technique of loss outcome indistinguishability. This implies that the hypothesis $\sign(p -1/2)$ is a strong agnostic learner. 
\end{itemize}

We are unaware of work that considered multiaccuracy as a learning primitive. In light of the above results, it would be natural to guess that multiaccuracy ought to imply (at least) weak agnostic learning. This would tightly characterize weak agnostic learning as being necessary and sufficient for all the multigroup fairness notion above. Perhaps surprisingly, the answer is more nuanced.

\subsection{Agnostic Learning and Multiaccuracy: With and Without Calibration}

 We start with the observation that weak learning is essentially equivalent to non-trivial square loss minimization, that beats random guessing. This observation is not entirely new, but will help elucidate several of our results.

Assume that $h$ is a $\beta$-weak learner. Then, as observed by \cite{hkrr2018}, taking a standard gradient step, the predictor $p(\x) = (1 + \beta h(\x))/2$ satisfies
\begin{align*}
    \E[(\y - p(\x))^2] \leq \E\lt[\lt(\y - \fr{2}\rt)^2\rt] - \frac{3\beta^2}{4}.
\end{align*}
In other words, weak agnostic learning enables non-trivial squared loss minimization, when compared to random guessing. 

In the other direction, we show that any function $h:\X \to \R$ which beats random guessing in squared error gives a weak agnostic learner.  Theorem \ref{thm:h-projection} shows that if $h:\X \to [0,1]$ satisfies
    \begin{align*}
         \E[(\y - h(\x))^2] \leq \E\lt[\lt(\y - \fr{2}\rt)^2\rt] - \gamma.
    \end{align*}
	 then $2h(x) -1$ is a $2\gamma$ weak learner. A similar result is proved in the work of \cite{BlumFJKMR94}.
With this characterization in hand, we can rephrase our question to ask when a multiaccurate predictor improves upon random guessing in terms of squared loss. 
More generally, for what values of $\alpha, \beta$ can one get an $(\alpha, \beta)$-weak learner for $\mC$ by post-processing a $\mC$-multiaccurate predictor?

\paragraph{When multiaccuracy does not imply agnostic learning.}

Motivated by Theorem \ref{thm:h-projection}, let us first focus on the hypothesis $q(x) = 2p(x) -1$ and understand when this fails as a weak learner. To frame this as a geometric question, let $q^*(x) = 2p^*(x) -1$. Assume there exists $c \in \mC$ so that $\la q^*, c\ra = \cor(\y, c) \geq \alpha$. By multiaccuracy, this also means that $\la q, c\ra \geq \alpha$, since multiaccuracy preserves the projection onto $\mC$. For what values of $\alpha$ is it possible that $\la q^*, q\ra =0$, so that $q$ is not a $\beta$-weak learner for $\mC$?

We will show that the answer is $\alpha \leq 1/2$.  Take $q, q^*$ to be vectors in $\pmo^n$, so that $\la q, q^*\ra = 0$. Let $d_H$ denote the normalized Hamming distance in this space. Let $c \in \pmo^n$ be the mid-point of $q, q^*$ in the Hamming metric. Since $d_H(q,c) = d_H(q^*, c) = 1/4$, we have $\la q^*, c\ra = \la q, c\ra = 1/2$. Thus we have a triple of vectors $q^*, q, c \in \pmo^N$ such that $\la q^*, c\ra = \la q, c \ra = 1/2$, whereas $\la q^*, q \ra = 0$. Translating this geometric example to a learning setting requires some care, but it conveys the intuition for why $1/2$ arises as the right threshold. 

 Theorem \ref{thm:ma-wal-} builds on this to show that there exists a distribution $\mD$, a predictor $p$ and a hypothesis class $\mC$ such that 
\begin{itemize}
\item $\opt(\mC) \geq 1/2$,
\item $p$ is $(\mC, 0)$-multiaccurate,
\item $\cor(\y, k(p)) =0$ for any $k:[0,1] \to [-1,1]$.
\end{itemize}
The key challenge is ruling out all post-processing functions $k(p)$, not just $2p-1$. Our construction achieves this by ensuring that for every predicted value $v$, $\E[\y|p(\x) = v] = 1/2$. Thus the predictor is {\em anti-calibrated} in the sense that even conditioned on its predictions, the bit $\y$ is uniformly random. This property ensures that no function $k(p)$ can beat random guessing.

Our construction in Theorem~\ref{thm:ma-wal-} makes careful use of the properties of majority and parity functions with respect to the uniform distribution. However, under the additional assumption of the existence of pseudorandom functions, we show in Theorem~\ref{thm:ma-wal-2} that for any reasonable class $\C$, a multiaccurate predictor cannot in general be post-processed to have non-trivial correlation with the labels, even when $\opt(\mC)$ is large. The key idea is based on making $p^*$ and $p$ both be averages of a function $f$ from $\C$ and a pseudorandom (for $\C$) function $g$, but with opposite signs. Thus, $p^* - p$ is essentially a pseudorandom function, which ensures multiaccuracy. However, even though $p^*$ has a high correlation with some $c \in \C$, $\cor(p^*, p) \approx 0$. 

\paragraph{When multiaccuracy implies (restricted) weak agnostic learning.}
Complementing this lower bound, Theorem \ref{thm:MA-to-meek} shows that $p$ being $\mC$-multiaccurate implies that $2p -1$ is an $(\alpha, 2\alpha -1)$ weak agnostic learner. Note that $\beta > 0$ only when $\alpha > 1/2$. This is in contrast with the standard usage of weak learning, where one typically thinks of $\alpha$ being close to $0$, hence we refer to this notion as restricted weak agnostic learning. In the realizable setting where $\alpha =1$, we get $\beta =1$. 

Intuition for why this is the right bound on $\beta$ can again be obtained from the geometric setting, where $q^*, q, c$ are vectors in $\pmo^n$. The condition $\la q, c\ra \geq \alpha$ translates to $d_H(q, c) \leq (1 - \alpha)/2$, and similarly for $d_H(q^*,c)$. By the triangle inequality, we must have
\[ d_H(q,q^*) \leq d_H(q,c) + d_H(q^*,c) \leq 1 - \alpha.\]
But this means that the correlation is lower bounded by 
\[ \la q, q^*\ra \geq 1 - 2(1 - \alpha) = 2\alpha -1.\]

While translating this intuition from the Boolean to the continuous setting is by no means immediate, we show that $2p -1$ is indeed a $\beta = 2\alpha -1$-weak learner. 
The key step is to show the following upper bound on the correlation obtainable using any hypothesis in $\mC$: if $p$ is $(\mC, \tau)$-multiaccurate, then 
\[\opt(\mC) \leq  \E_\mD[p(\x)p^*(\x) + (1 - p(\x))(1 - p^*(\x))] + \tau.\]

\paragraph{Calibrated multiaccuracy gives strong agnostic learning.}

We have seen that $\mC$-multiaccuracy alone might not give a weak learner. Similarly, calibration by itself need not give a weak learner, since predicting the constant $1/2$ might be calibrated (if $\y$ is balanced). This is a well-known problem with calibration as a standalone notion: its predictions might not be informative.

But we observe that this is the only way that it can fail.
In Lemma \ref{lem:cal-to-al}, we show that  for a calibrated predictor $p$,
\[ \cor(\y, \sign(2p -1)) \geq 2\E\lt[\lt|p(\x) - \fr{2}\rt|\rt] - 2\tau.\]
Thus a calibrated predictor that deviates noticeably from random guessing (saying $1/2$ all the time) gives rise to an agnostic learner. 
But if we add the assumption that $p$ is $\mC$-multiaccurate, and that there exists $c \in \mC$ so that $\cor(\y, c(\x)) \geq \alpha$, then $p$ can no longer predict $1/2$ all the time. This is proved in Lemma  \ref{lem:ma-to-al} which shows that for every $c \in \mC$,
\[ \E\lt[\lt|p(\x) - \fr{2}\rt|\rt] \geq \fr{2} \cor(\y, c(\x)) - \tau. \]
Intuitively, a hypothesis $c$ that is correlated with $\y$ provides a certificate that $\y|\x$ is not uniformly random. A $\mC$-multiaccurate predictor has to capture the correlations with $c$ accurately, which forces it to deviate from random guessing. Putting these insights together, Theorem \ref{thm:calma-to-sal} shows that if $p$ is both calibrated and $\mC$-multiaccurate, then $\sign(2p -1)$ is a strong agnostic learner.

We note that this result can also be shown using the loss OI technique of \cite{lossOI}, where it is applied to show agnostic learning for Boolean function classes $\mC$ (where correlation coincides with $0$-$1$ loss). But we find the proof above
more intuitive as it highlights how multiaccuracy and calibration complement each other. 

\subsection{Hardcore Sets of Optimal Density}

In the context of hardcore sets, we will restrict $\mC = \{c:\X \to \zo\}$ to be a class of Boolean functions, rather than bounded functions. We will assume that it contains the constant functions $\mathrm{0}, \mathrm{1}$ and is closed under complement. We will use $\mD_\X$ to denote a distribution over $\X$. Given a class of functions $\mC$, we are interested in studying how well the functions $c \in \C$ can approximate a target {\em hard} Boolean function $g$. For $\delta>0$, we say that $g$ is $(\mC, \delta)$-\emph{hard} on $\mD_{\X}$ if for all $c \in \mC$, $\Pr_{\mD_{\X}}[c(\x) = g(\x)] \leq 1-\delta$. We say a distribution is hardcore when $\delta$ becomes arbitrarily close to $1/2$, since this means that it is hard to predict the values of $g$ much better than random guessing.

A \emph{measure} on $\X$ is a function $\mu: \X \rightarrow [0,1]$, which is not identically $0$.  It defines a probability distribution $\bar{\mu}$ by re-weighting $\mD_\X$ by $\mu$. A hardcore measure $\mu$ is one which induces a hardcore distribution $\bar{\mu}$.
Given a measure $\mu: \X \rightarrow [0,1]$ and a base distribution $\mD_{\X}$ on $\X$, the \emph{density} of $\mu$ in $\mD_\X$ is given by
    \[
        \dns(\mu) = \E_{\x \sim \mD_{\X}}[\mu(\x)].
    \]
The existence of hardcore measure $\mu$ immediately implies that $g$ is $\delta = \dns(\mu)/2$ hard for $\mC$. Impagliazzo's seminal Hardcore Lemma (IHCL) proves a surprising converse to this statement, provided we assume that $g$ is $\delta$-hard not just for $\mC$, but for a larger class $\mC' =\mC_{t,q}$ containing $q$ {\em oracle gates} from $\mC$ and a total of $t$ logical gates. Since $q$ is the more important parameter for us (as $\mC$ grows in complexity), in this informal discussion we will focus on $q$ and ignore $t$. 

\begin{theorem}[IHCL informal statement, \cite{imp95, hol05}]\label{thm:ihcl-inf}
There exists $q = \poly(1/\eps, 1/\delta)$ such that if $g: \X \rightarrow \{0,1\}$ is $(\mC_{q}, \delta)$-hard on $\mD_{\X}$, then there is a measure $\mu$ satisfying:
\begin{itemize}
    \item {\bf Hardcore-ness:}  $g$ is $(\mC, 1/2- O(\epsilon))$-hard on $\bar{\mu}$.
    \item {\bf Optimal density:} $\dns(\mu)=2\delta -o(1)$.
\end{itemize}
\end{theorem}

The density parameter of $2\delta$ is optimal, since even a function that computes $g$ perfectly on the complement of $\mu$ will incur error $\delta$ from $\mu$. Impagliazzo's original proof gives density $\delta$ \cite{imp95}, the improvement to $2\delta$ is due to Holenstein \cite{hol05}. This is a cornerstone result in the average-case hardness of functions, and has led to a long line of work that explores connections to other areas, such as boosting \cite{ks03, fel09}, pseudorandomness \cite{ReingoldTTV08, TrevisanTV09} and more recently, to multigroup fairness \cite{TrevisanTV09, casacuberta2024complexity}. The last two are most relevant to us, so we discuss them in more detail.

\paragraph{IHCL with density $\delta$ from multiaccuracy.}
Trevisan, Tulsiani, and Vadhan showed that one can derive Impagliazzo's Hardcore Lemma from the existence of (what we now call) multiaccurate predictors. We consider the labelled distribution $\mD^g =(\x, g(\x))$ and let $p$ be a $(\mC, \epsilon\delta)$-multiaccurate predictor $p$ for $g$.  Their hardcore measure defined as $\ttvH(x) = |g(x) -p(x)|$ is simple and explicit.  Their analysis shows  $\dns(\ttvH) \geq \delta$, and  indeed (while we give a slightly sharper density analysis), one can construct instances where the $\dns(\ttvH) = \delta + o(1)$. 
The parameter $q =O(1/(\eps^2\delta^2))$ is governed by the number of calls to a (proper) agnostic learner for $\mC$ needed to achieve multiaccuracy with error $\tau  =\eps\delta$.

\paragraph{IHCL$++$ from multicalibration.}
Recently, Casacuberta, Dwork and Vadhan \cite{casacuberta2024complexity} showed that starting from a multicalibrated predictor, one can derive a stronger and more general version of the IHCL, which they call IHCL$++$.
Specifically, they show that given a multicalibrated partition, we can construct hardcore distributions of locally optimal density on each (large enough) piece of the partition.
As expected, given that multicalibration is stronger than multiaccuracy, one can recover the original IHCL from IHCL$++$ by ``gluing'' these local hardcore distributions.
This glued distribution is a global $2\delta$-dense hardcore measure.

While this result recovers the optimal hardcore density, the complexity of the class $\mC_q$ is now considerably larger than in \cite{TrevisanTV09}, we now get $q = O(1/(\eps^{12}\delta^6))$. This increased complexity comes in part because achieving $\tau$-multicalibration requires $O(1/\tau^6)$ calls to a (strong) agnostic learner for $\mC$, rather than the $O(1/\tau^2) $ calls needed for multiaccuracy.  We refer the reader to a detailed discussion at the end of Section \ref{sec:opt}.

\paragraph{Our result: IHCL from calibration and weighted multiaccuracy.}

We show how to construct hardcore distributions with optimal density $2\delta$ as in \cite{casacuberta2024complexity} but with the lower complexity bound of $q = O(1/(\eps^2\delta^2))$ for $\C_{q}$ as in \cite{TrevisanTV09}. Indeed, we give a simple and explicit form for such a measure, starting from a predictor satisfying calibration and a weighted variant of multiaccuracy that we call \emph{weighted multiaccuracy}. Both calibration and weighted multiaccuracy are implied by multicalibration, so this gives an alternate proof of the optimal density result of \cite{casacuberta2024complexity}. But crucially, the cost of obtaining both these conditions together (by many measures, including oracle calls to the weak learner) is not much more than that of multiaccuracy, which yields a considerably better bound on the parameter $q$.
To formally state our result, we define weighted multiaccuracy.

We will consider weight functions $w:[0,1] \to [1,2]$, which assign weights to predictions,  as is standard in the calibration literature \cite{GopalanKSZ22, utc1}. While the range $[-1,1]$ is more standard, this choice of range is appropriate for hardcore measures. The weight function that will be important for us is
\[ \wmax(p) = \frac{1}{\max(p, 1 - p)}.\]
Let $\mD$ be a distribution on $\X \times \zo$ whose marginal on $\X$ is $\mD_\X$.  For a weight function $w: [0,1] \to [1,2]$, we define the $(w, \mC)$-\emph{multiaccuracy error} of $p$ under the distribution $\mD$ as
\[ \MA_{\mD}(w, \mC, p) = \max_{c \in \mC} \Big| \E_{(\x, \y) \sim \mD}[c(\x)  w(p(\x)) (\y-p(\x))] \Big| .\]

Let us contrast weighted multiaccuracy to the notions of multiaccuracy and multicalibration. Multiaccuracy corresponds to the case where $w =1$, so it is a special case of weighted multiaccuracy. Weighted multiaccuracy specializes the definition of weighted multicalibration from \cite{GopalanKSZ22} to the case when the weight family consists of a single weight function.  Having to focus on a single weight function rather than a family translates to the complexity of computing a predictor satisfying weighted multiaccuracy being similar to that needed for multiaccuracy. 

With this definition in place, we can informally state our main result. 
Recall the distribution $\mD^g = (\x, g(\x))$ for $\x \sim \mD_\X$.

\begin{theorem}[IHCL from calibration and weighted multiaccuracy]
\label{thm:ihcl-inf}
Let $\eps, \delta > 0$, and $q =O(1/(\eps^2\delta^2))$. Assume that  $g: \X \rightarrow \{0,1\}$ is $(\mC_{q}, \delta)$-hard on $\mD_{\X}$. Let $p: \X \to [0,1]$ denote a predictor, and define the measure
\[ \hmH(x) = \wmax(p(x)) \cdot|g(x) - p(x)| = \frac{|g(x) - p(x)|}{\max(p(x), 1 - p(x))}.\]
Then:
\begin{itemize}
    \item {\bf Hardcore-ness from weighted multiaccuracy:} If $\MA_{\mD^g}(\wmax, \mC, p) \leq \eps\dns(\hmH)$, then the function $g$ is $(1/2 - \eps)$-hard for $\mC$ under the distribution $\bar{\mu}_{\mathsf{Max}}$.
    \item {\bf Optimal density from calibration:} If $p$ is $\tau$-calibrated, then $\dns(\hmH) = 2\delta  - O(\tau)$.
\end{itemize}
\end{theorem}

Here are some key advantages of our result and proof technique:
\begin{itemize}
\item We obtain the optimal density bound of $2\delta$ as in \cite{casacuberta2024complexity}, but under the assumption of $(\mC_q, \delta)$ hardness for $q =O(1/(\eps^2\delta^2))$ as in \cite{TrevisanTV09}.
\item  It clearly delineates the roles played by each of calibration and multiaccuracy in ensuring the resulting measure is a hardcore distribution with optimal density $2\delta$: calibration ensures that the distribution has the optimal density, while weighted multiaccuracy implies the hardcore property. 
\item Our technique applies to large class of measures, that we call balanced under calibration measures, which are of the form 
\[ \mu_w(x) = w(p(x)) \cdot|g(x) - p(x)| \] 
for a weight function $w:[0,1] \to [1,2]$ where $1 \leq w(p) \leq \wmax(p)$. This family includes $\ttvH$ which corresponds to $w =1$. We give a tight analysis of the density of all such measures assuming calibration. Our analysis reveals that $\dns(\ttvH)$ can sometimes be better than $\delta$.
\end{itemize}
Along the way, our argument fixes a gap in the gluing argument used to derive the IHCL with density $2\delta$ from IHCL$++$ in \cite{casacuberta2024complexity}, see the discussion at the end of Section \ref{sec:opt}.\footnote{Gluing is straightforward for hardcore sets, but more subtle for hardcore measures.}

\paragraph{Improving the density.}

We explain at an intuitive level how $\hmH$ improves over $\ttvH$ and gets optimal density. Consider a level set $S_v = \{x: p(x) = v\}$ of the predictor $p$ and partition it into $S_v^0$ and $S_v^1$ based on whether $g(x)$ is $0$ or $1$. $\ttvH$ assigns weight $1 - v$ to the points in $S_v^1$ and $v$ to points in $S_v^0$. If the predictor $p$ is calibrated for this level set, so that $\Pr[g(\x) = 1|\x \in S_v] = v$, then 
\[ \E[\ttvH(\x)|\x \in S_v] = v\Pr[g(\x)  =0|\x \in S_v] + (1 -v)\Pr[g(\x)  = 1|\x \in S_v] = 2v(1-v).\]
Crucially, the measure gives equal weight $v(1-v)$ to $S_v^0$ and $S_v^1$.

Now consider the measure $\hmH$. If $v < 1/2$ then $\hmH$ assigns weight $1$ to points in $S_v^1$, which are in the minority, and constitute a $v$ fraction of $S_v$. It assigns weight $v/(1 -v)$ to points in $S_v^0$, which are in the majority, and constitute a $1 - v$ fraction of $S_v$. Overall the weight given to each of $S_v^0$ and $S_v^1$ is $v= \min(v, 1 -v)$. When $v > 1/2$, the set $S_v^1$ has the majority of points, so $\hmH$ assigns weight $1$ to points in $S_v^0$, and weighs $S_v^1$ so that both have equal weight $\min(v, 1-v)$. 

Why is this a good weighting strategy? Consider the majority predictor $\tilde{p} = \ind{p \geq 1/2}$.\footnote{This predictor is the best-response if $p$ is Bayes optimal.} $\hmH$ assigns weight $1$ to the points where $\tilde{p}$ is incorrect. The rest of the weights are assigned so that $S_v$ is balanced, assuming that $p$ is in fact calibrated. This clearly ensures that $\tilde{p}$ is incorrect with probability $1/2$. We now want to argue both that $\hmH$ has good density and that no hypothesis in $\mC$ can do better than $1/2$. It is intuitive that if $\hmH$ is hard for $\tilde{p}$, which incorporates the distinguishing power of the class $\mC$,  then it should also be hard for $\mC$.

The analysis of the density follows that of \cite{casacuberta2024complexity}, with the added observation that this part only requires global calibration of $p$, not multicalibration. We assume that $\tilde{p} \in \mC_q$, so it must have error at least $\delta$. But the error of $\tilde{p}$ is exactly $\E[\min(p(\x), 1 - p(\x))]$. Whereas we have
\[ \dns(\hmH) = 2\E[\min(p(\x)), 1 - p(\x))], \]
so we get that $\dns(\hmH) \geq 2\delta$. A similar analysis of $\ttvH$ under the assumption of calibration shows 
\[ \dns(\ttvH) = \E[2p(\x)(1 - p(\x))]. \]
This is at least $\delta$, since $2v(1 - v) \geq \min(v, 1- v)$ for $v \in [0,1]$. If $p(\x)$ is far from $1/2$ with reasonable probability, then it gives a better bound than $\delta$. 

It remains to argue that the measure $\hmH$ is hardcore. In Lemma \ref{lem:wma-cor} we present a characterization of the prediction advantage of $\mC$ under the distribution $\bar{\mu}_w$ in terms of the $(w, \mC)$-weighted multiaccuracy error of $p$ under the distribution $\mD^g$. This tells us that if the multiaccuracy error is $\tau$, then the prediction advantage under $\bar{\mu}_w$ is $\tau/\dns(\mu_w)$. This proof technique is different from \cite[Lemma 2.17]{casacuberta2024complexity} who show
that if a function is indistinguishable from a constant function $v$ to $\mC$ on each level set $S_v$, then then one can obtain a hardcore set of density $2\min(v, 1 -v)$ within $S_v$. The required indistinguishability holds for every level set by the assumption of multicalibration. 
Our result can be seen as a simplification and extension of the hardness result in \cite{TrevisanTV09} to weighted measures.

Note that calibrated multiaccuracy does not give the stronger IHCL$++$ theorem proved in \cite{casacuberta2024complexity}; indeed that seems to require the full power of multicalibration. Also, the min-max based proof of Nisan as cited in~\cite{imp95} and the boosting-based proof of \cite{ks03} give $q$ which grows as $O(\log(1/\delta))$, but they do not get optimal density. The $1/\eps^2$ dependence on $\eps$ was recently showed to be necessary by \cite{blanc2024sample}.

\subsection{Multiaccuracy and Restricted Weak Agnostic Learning}

As discussed earlier and shown in Theorem~\ref{thm:MA-to-meek}, any $p$ that is multiaccurate for $\mC$ can be viewed as a $(2\alpha - 1)$-weak learner when $\alpha > 1/2$. Thus multiaccuracy by itself only yields a \emph{restricted} form of $(\alpha, \beta)$-weak agnostic learning, giving non-trivial learners only for $\alpha > 1/2$. Syntactically, the task of $(\alpha, \beta)$-weak learning gets harder as $\alpha$ decreases, since this corresponds to detecting lower correlation, and as $\beta$ increases, since this corresponds to learning a better hypothesis.

The agnostic boosting results of~\cite{kk09,fel09} tell us that increasing $\beta$ does not necessarily result in harder learning problems, provided weak learners exist for arbitrarily small values of $\alpha$. They show that an $(\alpha, \beta)$ weak learner implies a $(\gamma, \gamma - \alpha)$ weak learner for all $\gamma > \alpha$. This will improve on the trivial $(\gamma, \beta)$ weak learner provided $\gamma > \alpha + \beta$. In particular, this result shows that if weak agnostic learning is possible for arbitrarily small values of $\alpha$, then strong agnostic learning is possible.

The following natural questions arise:
\begin{itemize}
	\item \textbf{Boosting the $\alpha$ parameter}: Given an $(\alpha, \beta)$-weak learner, can one construct an $(\alpha^\prime, \beta^\prime)$-weak learner for $\alpha^\prime < \alpha$? If this is possible, it would give a boosting result for the $\alpha$ parameter in restricted weak agnostic learning. 
	\item \textbf{Multiaccuracy from Restricted Weak Agnostic Learning}: Given an $(\alpha, \beta)$-weak agnostic learner for $\mC$, can one get a $(\mC, \alpha')$-multiaccuracy for $\alpha' < \alpha$? \cite{hkrr2018} already show how to get $\alpha' = \alpha$.
\end{itemize}
The questions are related since if we can boost the $\alpha$ parameter, we can also reduce multiaccuracy error. We answer both questions in the negative, but the second comes with some loss in parameters. 

\paragraph{The $\alpha$-parameter for Restricted Weak Agnostic Learning cannot be boosted.} We make a standard cryptographic assumption that \emph{pseudorandom} functions cannot be distinguished from random functions with non-negligible advantage by \emph{efficient} algorithms;\footnote{In fact, these algorithms could be super-polynomial depending on how strong an assumption one is willing to make.} this, for example, follows from the existence of one-way functions~\cite{GoldreichGM86}. The result below is an informal version of Theorem~\ref{thm:noise-hard} (Parts~\ref{thm:noise-hard-easy} and~\ref{thm:noise-hard-hardness1}).

\begin{theorem} \label{thm:noise-hard-informal}
	Assuming the existence of a \emph{pseudorandom} family of functions, there exists a concept class $\mC$, such that, for every $\alpha \in (0, 1)$, there is a marginal distribution $\mD_\X$, so that,
	\begin{itemize}
		\item $\mC$ is \emph{efficiently} $(\alpha + \epsilon, \alpha)$-weak agnostically learnable under any $\mD$ with marginal $\mD_\X$, 
		\item $\mC$ is not \emph{efficiently} $(\alpha, \beta)$-weak learnable for distributions $\mD$ with marginal $\mD_\X$. 
	\end{itemize}
\end{theorem}

The domain $\X$ is divided into two disjoint parts $E$ and $F$; each hypothesis consists of a pseudorandom function $f_r$ and its key $r$. The key $r$ is represented in $E$ and the hypothesis on $F$ is just the pseudorandom function $f_r$. In the absence of noise, this class is not hard to learn, as the key can be constructed from the examples in $E$ and then the function $f_r$ can be identified exactly; on the other hand, if there is sufficient noise to make the labels on $E$ completely random, then the learning problem becomes equivalent to learning pseudorandom functions. In order to get a tight gap, instead of representing $r$ directly in $E$, we use a representation using error-correcting codes.
In particular, if we use list-decodable codes that can tolerate error up to slightly less than $1/2$ (e.g.\cite{GuruswamiS00}), then we obtain the tight characterization given above.

\paragraph{Multiaccuracy from Restricted Weak Agnostic Learning.} As discussed
earlier, \cite{hkrr2018} already show that access to an $(\alpha, \beta)$-weak
agnostic learning algorithm is sufficient for constructing a predictor $p$ that
is $\alpha$-multicalibrated for $\mC$. Using the same construction as in
Theorem~\ref{thm:noise-hard-informal}, we show (Theorem~\ref{thm:noise-hard},
Part~\ref{thm:noise-hard-hardness2}) that it is not possible to get
$\tau$-multiaccuracy for $\tau < \alpha/2 - \eps$ using only an $(\alpha,
\beta)$-weak agnostic learning algorithm. Intuitively this factor of $2$ comes
as $\cor(\y, c) - \cor(p, c) = 2 \E[c(\x) (\y - p(\x))]$;%
\footnote{As an aside, we note that even if we only had $(1, \beta)$-weak learning, corresponding to \emph{realizable} weak learning, achieving $1/2$ multiaccuracy is straightforward by setting $p(\x) = 1/2$ for all $\x$.}
thus if $\cor(\y, c) = \alpha + \eps$, we get that $\cor(p, c) \geq \eps$. As in
the Theorem~\ref{thm:noise-hard-informal} described above, the distribution
$\mD$ is chosen so that $\y$ is random noise on $E$, so learning directly with
access to the labels is as hard as learning a pseudorandom function. On the
other hand, depending on whether the non-negligible correlation $\cor(p, c) \geq
\eps$ comes from $E$ or $F$: we can either decode the key $r$ and hence learn
$f_r$, or we have that $p$ itself has non-negligible correlation with the
pseudorandom function $f_r$. In either case, we violate the cryptographic
assumption.

\paragraph{The spectrum of weak agnostic learning.}
Figure \ref{fig:WAL} summarizes the spectrum of what we know about the feasibility of weak agnostic learning. The region of interest is $\beta \leq \alpha$. Going right or down, the problem becomes easier, whereas going up or left is harder. Green indicates regions that are known to be easy in general (under the right learning assumptions), while Red indicates regions that are hard (in the worst case over all $\mC$).

\begin{itemize}
    \item The Orange line bounds the region that is feasible if we have access to a $(\mC, 0)$-multiaccurate predictor. 
    This follows from our Theorem~\ref{thm:MA-to-meek}.
    If we also assume calibration, the entire region $\beta \leq \alpha$ is feasible.
    \item The Blue line with slope $1$ follows from the boosting results of \cite{kk09, fel09}.  
    \item The Red region indicates that there is no black-box reduction from $(\alpha, \beta)$ weak agnostic learning to $(\alpha - \eps, \beta')$ weak agnostic learning, assuming the existence of one-way functions.
    This follows from our Theorem~\ref{thm:noise-hard}.
\end{itemize}

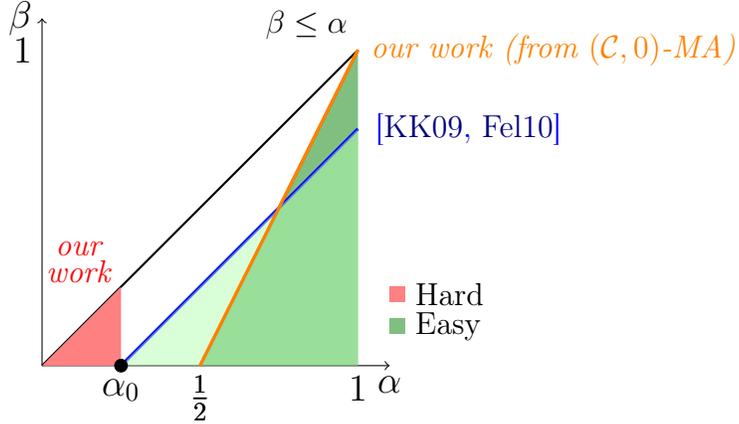
\begin{figure}
\hspace{4.7cm}
\begin{tikzpicture}[scale=4.2]
    \draw[->] (0,0) -- (1.1,0) node[below] {\Large $\alpha$};
    \draw[->] (0,0) -- (0,1.1) node[left] {\Large $\beta$};
    
    \node[below] at (1,0) {\Large $1$};
    \node[left] at (0,1) {\Large $1$};
    
    \draw[thick] (0,0) -- (1,1) node[above left] {\large $\beta \leq \alpha$};
    
    \draw[very thick, orange] (0.5,0) -- (1,1);
    \node[below] at (0.5,0) {\Large $\frac{1}{2}$};
    
    \fill[green!50!black,opacity=0.5] (0.5,0) -- (1,1) -- (1,0) -- cycle;
    
    \def\alphaZero{0.25}
    \def\betaZero{0}
    
    \filldraw[black] (\alphaZero,\betaZero) circle (0.01);
    
    \node[below] at (\alphaZero,-0.01) {\Large $\alpha_0$};
    
    \draw[very thick, blue] (\alphaZero,\betaZero) -- (1, {\betaZero + (1 - \alphaZero)});
    
    \fill[green!30,opacity=0.5] (\alphaZero,\betaZero) -- (1, {\betaZero + (1 - \alphaZero)}) -- (1,0) -- (\alphaZero,0) -- cycle;
    
    \fill[red!50] (0,0) -- (\alphaZero,\alphaZero) -- (\alphaZero,0) -- cycle;
    
    
    \node[red] at (\alphaZero - 0.13, \alphaZero + 0.12) {\large \textit{our}};
    \node[red] at (\alphaZero - 0.13, \alphaZero + 0.05) {\large \textit{work}};
    
    \node[orange] at (1.62,1) {\large \textit{our work \large (from $(\C, 0)$-MA)}}; 
    
    \node[blue] at (1.35, {\betaZero + (1 - \alphaZero)}) {\large \cite{kk09, fel09}};

    \filldraw[black] (\alphaZero,\betaZero) circle (0.02);

    \draw[very thick, orange] (0.5,0) -- (1,1);
    \node[below] at (0.5,0) {\Large $\frac{1}{2}$};

    \fill[red!50] (1.1,0.2) rectangle (1.15,0.25);
    \node[right] at (1.15,0.225) {\large Hard};
    
    \fill[green!50!black,opacity=0.5] (1.1,0.1) rectangle (1.15,0.15);
    \node[right] at (1.15,0.125) {\large Easy};

\end{tikzpicture}
    \caption{The spectrum of weak agnostic learning.}
    \label{fig:WAL}
\end{figure}

%% file: 2-prelims.tex
\section{Notation \& Preliminaries}\label{sec:preliminaries}

We introduce the notation and background that is used throughout the paper in this section. Further notation and background is introduced in sections as required. 
We consider the setting of supervised learning where $\mD$ is a distribution on pairs $(\x, \y)$ where $\x$ is a feature vector from a space $\X$ and  $\y \in \zo$ is a label. We denote by $\mD_\X$ the marginal distribution over $\X$.
Boldface typically refers to random variables.

A predictor is a function $p:\X \to [0,1]$ which associates every feature vector $\x$ with an (estimated) probability $p(\x)$ that the label $\y|\x$ is $1$. The Bayes optimal predictor is given by $p^*(\x) = \E[\y|\x]$, and represents the {\em ground truth}. 

\subsection{Agnostic Learning}\label{sec:prelims-agnostic-learning}

We consider a hypothesis class $\mC$, where each $c \in \mC$ is a function from $\X \rightarrow [-1, 1]$. 

\begin{definition}
\label{def:cor-loss}
The \emph{correlation} of a function $c : \X \rightarrow [-1, 1]$ with respect to a distribution $\mD$ over $\X \times \zo$, is defined as $\cor_{\mD}(\y, c) = \E_{(\x, \y) \sim \mD}[ (2 \y - 1) c(\x)]$. 
\end{definition}

The transformation $\y \to 2\y -1$ maps $\zo$ to $\pmo$. Hence the correlation measures the correlation between $c$ and $\pmo$ labels associated with $\y$. Given a predictor $p : \X \to [0, 1]$, we will use the function $2p - 1$ which has range $[-1, 1]$ when considering correlation. The notion of weak agnostic learning we use is along the lines of that defined in~\cite{KalaiMV08, kk09, fel09}.

\begin{definition}[Weak agnostic learner] \label{defn:alpha-beta-WAL}
    For $\alpha \geq \beta > 0$, an $(\alpha, \beta)$-\emph{weak agnostic learner} for $\mC$ with marginal distribution $\mD_\X$ is an algorithm that when given access to random examples from any $\mD$ whose marginal over $\X$ is $\mD_\X$, if there exists $c \in  \mC$ such that 
	 \[ \cor_{\mD}(\y, c(\x)) \geq \alpha,\]
    returns a hypothesis $h$ such that with probability at least $0.99$,%
	 \footnote{As the algorithm has access to random examples, for any $\rho > 0$, we can always boost the success probability to guarantee that for any $0< \tau \leq \beta$, the output of the algorithm will satisfy, with probability at least $1 - \rho$, $\cor_{\mD}(\y, h(\x)) \geq \beta - \tau$. The increase in sample complexity and runtime bounded by a factor that is polynomial in $1/\tau, \log(1/\rho)$.}
	 \[ \cor_{\mD}(\y, h(\x)) \geq \beta.\]
    We further say that $h$ is a \emph{proper} weak agnostic learner if $h \in \C$, and \emph{improper} otherwise.
\end{definition}
Weak agnostic learning is a broad definition that covers a wide range of learners.
\begin{itemize}
	\item A weak agnostic learner gets more powerful as $\alpha$ decreases and $\beta$ increases. In other words, an $(\alpha, \beta)$-weak learner implies an $(\alpha', \beta')$-weak learner for $\alpha' \geq \alpha$ and  $\beta' \leq \beta$. 
	\item  The most powerful weak agnostic learner is one that can handle any positive $\alpha$, and has $\beta = \poly(\alpha)$. Colloquially weak agnostic learning is sometimes used to refer to this setting (where $\alpha$ approaches $0$), but as results on agnostic boosting show \cite{kk09, fel09}, this is in fact equivalent to strong agnostic learning.%
		\footnote{We would expect the sample complexity and computational complexity to grow as an (ideally polynomial) function of $1/\beta$.}
	\item Intuitively, weak agnostic learning gets easier as $\alpha$ increases. The least powerful notion is a learner that only works in the realizable case when $\alpha = 1$. 
\end{itemize}

\begin{definition}[Strong agnostic learner]
	For a hypothesis class $\mC$, a \emph{strong agnostic learner} for $\mC$ under the marginal distribution $\mD_\X$ is an algorithm that given any $\eps > 0$ as the target error parameter and random samples from any $\mD$ with marginal distribution $\mD_\X$ returns a hypothesis $h: \X \to [-1,1]$ such that with probability at least $0.99$,
	\[ \cor_{\mD}(\y, h(\x)) \geq \max_{c \in \mC} \cor_{\mD}(\y, c(\x)) - \eps.  \]
\end{definition}
The failure probability can be made arbitrarily small using standard techniques. A strong agnostic learner yields an $(\alpha, \alpha
-\eps)$ weak agnostic learner for all $\eps > 0$ and $\alpha \in (\eps, 1]$.

Agnostic boosting \cite{kk09, fel09} tells us that strong agnostic learning reduces to weak agnostic learning for small $\alpha$. 
The result below follows by appealing to Theorem 1 of~\cite{kk09}. In particular, it is straightforward to see that an $(\alpha, \beta)$ weak learner as in Definition~\ref{defn:alpha-beta-WAL} is an $(\alpha \beta, \beta, 0.01)$ weak learner as defined in their work.

\begin{theorem}[Agnostic Boosting, \cite{kk09}]\label{thm:kk09-boosting}
    An $(\alpha, \beta)$-weak agnostic learner for $\mC$ can be boosted to a strong agnostic learner under the same marginal distribution on $\X$ to obtain a hypothesis $h$ that satisfies with probability at least $1 - \rho$,
	 \[ \cor_{\mD}(\y, h(\x)) \geq \max_{c \in \mC} \cor_{\mD}(\y, c(\x)) - \alpha - \eps \] 
	 for any $\eps > 0$, using $\poly(1/\eps, 1/\beta, 1/\rho)$ calls to the weak agnostic learner.
\end{theorem}

We can interpret this result as follows: provided we have an $(\alpha, \beta)$-weak agnostic learner for some $\alpha$, then for any $\alpha^\prime \geq \alpha$, we can obtain an $(\alpha^\prime, \alpha^\prime - \alpha - \epsilon)$-weak agnostic learner for any $\epsilon > 0$; furthermore such a learner can be obtained using at most $\poly(1/\epsilon, 1/\beta, 1/\rho)$ calls to the $(\alpha, \beta)$-weak agnostic learner. In particular, if we can design an $(\alpha, \beta)$-weak agnostic learner for any $\alpha > 0$, then we obtain (improper) strong agnostic learning. 

\subsubsection*{Weak Learning as Squared Loss Minimization}

We observe that weak learning is essentially equivalent to non-trivial squared loss minimization that beats random guessing. This observation is not entirely new, but will help elucidate several of our results.

Assume that we have a function $h: \X \to [-1,1]$ so that $\cor(\y, h(\x)) = \E[h(\x)(2\y - 1)] \geq \beta$. For instance, $h$ could be the output of an $(\alpha, \beta)$-weak agnostic learner. Then \cite{hkrr2018} shows that the update
$p(\x) = (1 + \beta h(\x))/2$ satisfies
\begin{align*}
    \E[(\y - p(\x))^2] \leq \E\lt[\lt(\y - \fr{2}\rt)^2\rt] - \frac{3\beta^2}{4}.
\end{align*}
In other words, weak agnostic learning enables non-trivial squared loss minimization, when compared to random guessing. 

In the other direction, we show that any function $h:\X \to \R$ which beats random guessing in squared error gives a weak agnostic learner. A similar result is proved in the work of \cite{BlumFJKMR94}. 

\begin{theorem}
\label{thm:h-projection}
    Let $h:\X \to \R$ satisfy
    \begin{align*}
         \E[(\y - h(\x))^2] \leq \E\lt[\lt(\y - \fr{2}\rt)^2\rt] - \gamma.
    \end{align*}
    Define  the function $\bar{h}$ which truncates $h$ to $[0,1]$, and let $g(x) = 2\bar{h}(x) -1 \in [-1,1]$. Then $\cor(\y, g(\x)) \geq 2\gamma$.
\end{theorem}
\begin{proof}
    Since truncation to the range $[0,1]$ only reduces the squared loss, we have
    \begin{align}
    \label{eq:sq-reduce}
         \E[(\y - \bar{h}(\x))^2] \leq \E\lt[\lt(\y - \fr{2}\rt)^2\rt] - \gamma.
    \end{align}
    We can write
    \begin{align*}
         \E[(2\y - 2\bar{h}(\x))^2]  &= \E[(2\y - 1 - g(x))^2]\\
													&= \E[(2\y - 1)^2] - 2\E[g(\x)(2\y -1)] + \E[g(\x)^2].
    \end{align*}
    Rearranging terms, we get
    \begin{align*}
 		 \E[g(\x)(2\y -1)] &= \frac{1}{2} \lt( \E[(2\y - 1)^2] - \E[(2\y - 2\bar{h}(\x))^2] + \E[g(\x)^2] \rt)\\ 
 									& \geq 2 \lt(\E\lt[\lt(\y - \fr{2}\rt)^2\rt]  - \E[(\y - \bar{h}(\x))^2]\rt)\\ 
 									& \geq 2 \gamma.
    \end{align*}
    where we use Equation \eqref{eq:sq-reduce}. The claim follows since the LHS is equal to $\cor(\y, g(\x))$.
\end{proof}

In essence, these results tell us that weak agnostic learning is equivalent to learning a predictor whose squared loss beats random guessing. This is reminiscent of the result of \cite{kothari2017agnostic} who show that weak agnostic learning for a class $\mC$ is equivalent to distinguishing labelled distributions where $\cor(\y, c(\x)) \geq \alpha$ for some hypothesis $c \in \mC$ from the distribution where $\y$ consists of random bits (since in the latter case, random guessing is in fact optimal).

\subsection{Multigroup Fairness Notions} \label{sec:prelims-multigroup-fairness}

We start by recalling the basic notions of accuracy in expectation and calibration.

\paragraph{Accuracy in expectation and calibration.}

A predictor $p$ is (perfectly) accurate in expectation if $\E[\y] = \E[p(\x)]$. Accordingly, we define the 
{\em expected accuracy error} of a predictor as
\begin{align}
\label{eq:acc-exp}
   \EAE(p) = \abs{\E[\y] - \E[p(\x)]}.
\end{align}
A predictor $p$ is calibrated if $\E[\y|p(\x)] = p(\x)$ (we assume $p(\x)$ has a discrete range for simplicity). Accordingly, we define the 
{\em expected calibration error} of a predictor as
\begin{align}
\label{eq:acc-exp}
   \ECE(p) = \E\lt[\abs{\E[\y|p(\x)] - p(\x)}\rt].
\end{align}
 We say $p$ is $\tau$-calibrated under $\mD$ if $\ECE(p, \mD) \leq \tau$.
The following characterization of $\ECE$ in terms of bounded auditing functions is standard (see e.g \cite{GopalanKSZ22}). 
\begin{align}
\label{eq:alt-ece} 
    \ECE(p, \mD) = \max_{v:[0,1] \to [-1,1]} \E_{\mD}[v(p(x))(y - p(x))].
\end{align}
We now define the notions of multiaccuracy \cite{TrevisanTV09, hkrr2018}, calibration \cite{Dawid}, calibrated multiaccuracy \cite{lossOI} and multicalibration \cite{hkrr2018}. In this subsection, we assume that $\mC$ is a class of functions, $c : \X \rightarrow \zo$; this is more natural as these functions are viewed as group membership.

\begin{definition} \label{defn:multiaccuracy}
    We say that a predictor $p$ is $(\mC, \tau)$-\emph{multiaccurate} for a distribution $\mD$ if 
	 \[ \max_{c \in \mC}\abs{\E_{\mD}[c(\x)(\y - p(\x))]} \leq \tau. \]
    We say the predictor is $\tau$-\emph{calibrated} for $\mD$ if
    \[ \E[|\E[\y|p(\x)] - p(\x)]|] \leq \tau.\]
    We say that $p$ is $(\mC, \tau)$-\emph{multiaccurate and calibrated} if it is both $(\mC, \tau)$-multiaccurate and $\tau$-calibrated.
    Further, we say that $p$ is $(\mC, \tau)$-\emph{multicalibrated} if
    \[ \max_{c \in \mC}\E[|c(\x)(\E[\y|p(\x)] - p(\x))|] \leq \tau.\]
\end{definition}

We will assume that the class $\mC$ always contains the constant $\mathbf{1}$ function; note that this is simply requiring that (approximate) accuracy in expectation hold. For a function $c : \X \to \zo$, let $\bar{c} : X \to [-1, 1]$ denote the function, $\bar{c}(\x) = 2 c(\x) - 1$. Then if $p$ satisfies $(\mC, \tau)$-multiaccuracy for distribution $\mD$, we have
\begin{align*}
	\abs{\E_{\mD}[ \bar{c}(\x) (\y - p(\x))]} &= \abs{\E_{\mD}[(2 c(\x) - 1) (\y - p^*(\x))]} \\
	&\leq 2 \abs{\E_{\mD}[ c(\x) (\y - p^*(\x)]} + \abs{\E_{\mD}[1 \cdot( \y - p^*(\x))]} \\
&\leq 3 \tau.
\end{align*}

Thus, if $\bar{\mC}$ is the hypothesis class representing the $\{\pm 1\}$ versions of the functions in $\mC$, then $p$ is $(\bar{\mC}, 3 \tau)$-multiaccurate with respect to $\mD$ as per Definition~\ref{defn:multiaccuracy}. As a result, we will move freely between the $\zo$ and $\{ \pm 1\}$ versions of functions.

Calibrated multiaccuracy is an intermediate notion between multiaccuracy and multicalibration; see \cite{lossOI} for more details. For any $\tau > 0$, it is known that having an $(\tau, \poly(\tau))$-weak agnostic learner for $\mC$ is sufficient to learn a predictor that $(\mC, \tau)$-multicalibrated \cite{hkrr2018, omni}. Indeed, if one only wants $(\mC, \tau)$-multiaccuracy, then there is a simple algorithm which performs least-squares regression over $\mC$ via gradient descent, using the weak agnostic learner at each stage to find a gradient update \cite{lossOI}. One could ask if the reverse connection holds: is having an algorithm that learns a $\mC$-multiaccurate predictor sufficient to get a weak agnostic learner for $\mC$? We provide some answers to this question in Section~\ref{sec:WAL-MA}.

%% file: 3-wal-ma.tex
\section{From Multiaccuracy to Agnostic Learning}\label{sec:WAL-MA}

As discussed in Section~\ref{sec:preliminaries}, it is known that having a $(\tau, \poly(\tau))$-weak agnostic learner for $\mC$ is sufficient learn a predictor that is $(\mC, \tau)$-multicalibrated with  $\poly(1/\tau)$ oracle calls to the weak learner~\cite{hkrr2018, omni}. Indeed, if one only wants $(\mC, \poly(\tau))$-multiaccuracy, there is a simpler algorithm which performs least-squares regression over $\mC$ via gradient descent, using the weak agnostic learner at each stage to find a gradient update~\cite{lossOI}. Here we consider the reverse question: does a $(\mC, \tau)$-multiaccurate predictor yield some form of agnostic learning? 

\subsection{Multiaccuracy and Weak Agnostic Learning} \label{sec:MA-to-WAL}

Our first theorem shows that a multiaccurate predictor $p$ by itself cannot be transformed into even a weak agnostic learner, even when there exists a function in $\mC$ that has $\cor_{\mD}(\y, c) = 1/2$. 

\begin{theorem}
\label{thm:ma-wal-}
There exists a class $\mC$, a distribution $\mD$ on $\X \times \zo$, and a predictor $p: \X \to [0, 1]$ such that
\begin{itemize}
    \item $p$ is $(\mC, 0)$-multiaccurate. 
    \item For any $k:[0,1] \to [-1,1]$, $k(p)$ is not a $(1/2, \beta)$-weak agnostic learner for $\mC$ for any  $\beta > 0$. 
\end{itemize}
\end{theorem}
\begin{proof}
We consider the uniform distribution on the domain $\pmo^n$. Let $\E[\y|\x] = p^*(\x)$ for $p^*: \pmo^n \to [0,1]$. Every such function specifies a joint distribution $\mD$ on $\pmo^n \times \zo$. 

Let $\mC = \{c: \pmo^{n-1} \to [-1,1]\}$ be the set of all functions of the first $n-1$ bits. Now define the predictor $p$ as $p(x_1,\ldots, x_n) = p^*(x_1,\ldots,x_{n-1}, -x_n)$. If we write the multilinear expansion of $p^*$ as
\begin{align*} 
p^*(x_1,\ldots, x_n) = p^*_0(x_1,\ldots, x_{n-1}) + x_np^*_1(x_1, \ldots, x_{n-1})
\end{align*}
then we have
\begin{align*} 
p(x_1,\ldots, x_n) &= p^*_0(x_1,\ldots, x_{n-1}) - x_np^*_1(x_1, \ldots, x_{n-1})\\
\text{hence} \ p^*(x) - p(x) &= 2x_np^*_1(x_1,\ldots, x_{n-1})
\end{align*}

It follows that for any function $c(x_1,\ldots, x_{n-1})$ of the first $n-1$ bits,
\begin{align*} 
    \E[c(\x)(\y - p(\x))] = \E[c(\x)(p^*(\x) - p(\x))] = \E[c(\x)2\x_np^*_1(\x)] =0
\end{align*}
since both $c$ and $p^*_1$ only depend on the first $n-1$ bits. So $p$ is $(\mC, 0)$-multiaccurate for $\mD$.

We now construct examples of distributions where $k(p)$ does not give non-trivial correlation for any $k : [0, 1] \to [-1, 1]$. Define $\maj: \pmo^3 \to \zo$ as
\[ \maj(x_1,x_2,x_3) = \begin{cases} 1 & x_1 + x_2 + x_3 > 0\\
0 & \text{otherwise}
\end{cases}
\]
For any pair of distinct indices $i, j \in [n-1]$ let us denote the distribution $\mD_{ij}$ where the marginal over $\x$ is uniform and the conditional label distribution $\y|\x$ is defined by the Bayes optimal predictor
\[ p^*_{ij}(\x) = \maj(x_i, x_j, x_n). \]
Then consider the predictor 
\[ p_{ij}(\x) = \maj(x_i, x_j, -x_n). \]
The key observation is that while $p_{ij}$ is multiaccurate, its predictions do not tell us much about $p^*_{ij}$. 
Indeed we have
\[ \Pr[p^*_{ij}(\x) = p_{ij}(\x)] = \Pr[\x_i = \x_j] = \fr{2}.\]
\[ \Pr[p^*_{ij}(\x) \neq p_{ij}(\x)] = \Pr[\x_i \neq \x_j] = \fr{2}.\]
A simple analysis shows that
\begin{align}
	\E[(2\y -1)| p_{ij}(\x) = 0] &= \E[(2\y -1) |p_{ij}(\x) = 1] = 0. \label{eqn:antical}
\end{align}
Hence for any $k:[0,1] \to [-1,1]$, 
\begin{align*} 
	\cor_{\mD_{ij}}(\y, k\circ p_{ij}(\x)) &= \frac{k(0)}{2}\E[(2\y -1)|p_{ij}(\x)= 0] + \frac{k(1)}{2}\E[(2\y -1)|p_{ij}(\x)= 1] = 0 .
\end{align*}
Although no function $k\circ p_{ij}$ has non-zero correlation with $p^*_{ij}$, the dictator function $c_{i} \in \mC$ defined as $c_i(\x) = x_i$ has correlation $1/2$ since
\begin{align*}
	\E[(2 \y - 1) x_i| x_i = 1] &= \E[(2 \y - 1) x_i| x_i = -1] = \frac{3}{4} - \frac{1}{4} = \frac{1}{2}.
	\intertext{Thus,}
	\cor_{\mD_{ij}}(\y, c_{i}(\x)) &= \frac{1}{2} \E[(2\y - 1) x_i | x_i = 1] + \frac{1}{2} \E[(2\y - 1) x_i | x_i = -1] = \frac{1}{2}.
\end{align*}
This proves that a black-box reduction from $(C, 0)$-multiaccuracy cannot give a $(1/2, \beta)$ weak agnostic learner for any $\beta > 0$. 
\end{proof}
We observe that Eqn.~\ref{eqn:antical} in the proof above shows that $p$ is
maximally \emph{anti-calibrated}, in the sense that $\E[\y | p(\x) = v] =
1/2$ for any $v$ in the range of $p$. Such a predictor is clearly no better than random guessing in terms of squared loss. This construction also provides a counterpoint to Lemma~\ref{lem:cal-to-al} which shows that any non-trivial
\emph{calibration} enforces some correlation with the target~$\y$.  

Our next result shows that assuming the existence of pseudorandom functions that fool the class $\C$, a very general result can be established --- namely that, in general, no post-processing of a multiaccurate predictor yields a weak agnostic learner. 

\begin{theorem}\label{thm:ma-wal-2} Let $\C$ be a hypothesis class of functions with range $\{-1, 1\}$ and $\mD_\X$ be a distribution over $\X$. Suppose there exists an efficiently computable function $g : \X \rightarrow \{0, 1\}$ such that,
	\[ \forall c \in \C, \forall v \in \{-1, 1\}, \quad \left|\Pr_{\mD_\X}[ g(\x) = 1 | c(\x) = v] - \frac{1}{2} \right| \leq \frac{\tau}{2}. \]
	Then there exists a distribution $\mD$ over $\X \times \{0, 1\}$ with the marginal distribution over $\X$ being $\mD_\X$, and an efficiently computable predictor $p : \X \rightarrow [0, 1]$ such that,
	\begin{itemize}
		\item $p$ is $(\C, \tau)$-multiaccurate for $\mD$. 
		\item For any $k : [0, 1] \rightarrow [-1, 1]$, $k(p)$ is not a $(q- \tau, \beta)$ agnostic learner for $\C$ for any $\beta > \tau$, where 
			\[ q = \max_{c \in C} \left\{\min\{ \Pr_{\mD_\X} [c(\x) = 1], \Pr_{\mD_\X}[c(\x) = -1] \} \right\}. \]
	\end{itemize}
\end{theorem}
\begin{proof}
	Fix some $f \in \C$ and let $q = \Pr_{\mD_\X}[f(\x) = 1]$. Assume without loss of generality that $0 < q \leq \frac{1}{2}$. Let $\hat{f} : \X \rightarrow \{-1, 1\}$ be a (randomized) function such that
	\begin{align*}
		\hat{f}(x) &= 
		\begin{cases}
			1 & \text{if } f(x) = 1, \\
			1 & \text{with probability } \frac{1 - 2q}{2(1 - q)}  \text{ if } f(x) = -1, \\
			-1 & \text{with probability } \frac{1}{2(1 - q)}  \text{ if } f(x) = -1. 
		\end{cases}
	\end{align*}

	We have defined $\hat{f}$ to be a $\pm 1$-valued function that is a balanced version of $f$; in particular, $\Pr_{\mD_\X}[\hat{f}(\x) = 1 ] = \frac{1}{2}$. Note that
	\begin{align}
		\E_{\mD_\X}[f(\x) \hat{f}(\x)] &= \E_{\mD_\X}[f(\x) \hat{f}(\x) | f(\x) = 1] \cdot q + \E_{\mD_\X}[f(\x) \hat{f}(\x) | f(\x) = -1] \cdot (1-q) = 2q. \label{eqn:ma-wal-2-connect-f-fhat}
	\end{align}

	Let $\hat{g} : \X \rightarrow \{-1, 1\}$ be defined as the $\pm 1$-valued version of $g$, i.e. $\hat{g}(x) = 2 g (x) - 1$. Define the distribution $\mD$ with the marginal distribution $\mD_\X$ over $\X$ such that,
	\[ \E[\y | \x] = p^*(\x) = \frac{2 + \hat{f}(\x) + \hat{g}(\x)}{4}. \]
	
	Finally, define a predictor $p : \X \rightarrow [0, 1]$ as
	\[ p(x) = \frac{2 + \hat{f}(x) - \hat{g}(x)}{4}. \]

	In order to check that $p$ is $(\C, \tau)$-multiaccurate for $\mD$, we can see that for any $c \in C$,
	\begin{align*}
		\E_{\mD_\X}[c(\x) (p^*(\x) - p(\x))] &= \frac{1}{2} \E_{\mD_\X}[ c(\x) \hat{g}(\x)] \leq \tau. 
	\end{align*}
	The latter inequality follows from the condition on $g$ in the statement of the theorem. 

	Similary, we can check that,
	\begin{align*}
		\E_{\mD_\X}[c(\x) (p^*(\x) - p(\x))] &= \frac{1}{2} \E_{\mD_\X}[ c(\x) \hat{g}(\x)] \geq -\tau. 
	\end{align*}
	Thus, $p$ is $(C, \tau)$-multiaccurate.

	We observe that $p(x) \in \left\{0, \frac{1}{2}, 1 \right\}$. For any such $v$ in the range of $p$, as we have picked $f \in C$, the condition on $g$ in the theorem statement shows that $|\cor(\y, p(\x) | p(\x) = v)| \leq \tau$. 

	On the other hand, we can also observe that (using Eqn.~\eqref{eqn:ma-wal-2-connect-f-fhat}),
	\begin{align*}
		\cor(\y, f) &= \E[(2\y - 1) f(\x)] \\
		 				&= \E[(2 p^*(\x) - 1) f(\x)] \\
						&= \frac{1}{2} \E[(\hat{f}(\x) + \hat{g}(\x)) f(\x)] \\
						&\geq q - \tau.
	\end{align*}

	Thus as long as $q > \tau$, there is a function in $\C$ with non-trivial correlation with $\y$, but no post-processing of $p$ yields a weak agnostic learner with correlation better than $\tau$. (Note that if $q$ is almost $0$ for every $f \in \C$, then $\C$ only contains almost constant (w.r.t $\mD_\X$) functions.)
\end{proof}

\begin{remark} The existence of a $g$ satisfying the conditions in the statement of Theorem~\ref{thm:ma-wal-2} with $\tau$ being negligible follows from the existence of one-way functions as long as $\C$ is contained in the class of polynomial-sized circuits~\cite{GoldreichGM86}. We have assumed above that $p$ can be a randomized predictor; again this assumption can be removed assuming the existence of one-way functions. Also note that if every $c \in \C$ is balanced, then we get $q = 1/2$; this was the case with the class used in Theorem~\ref{thm:ma-wal-}.
\end{remark}

The following theorem shows that Theorem~\ref{thm:ma-wal-} is tight in the sense that whenever the optimal correlation is noticeably greater than $1/2$, then a simple transformation $p \mapsto 2p -1$, does indeed provide a weak agnostic learning guarantee.

\begin{theorem}\label{thm:MA-to-meek}
For any $\C, \mD$, one can post-process a $(\mC, \tau)$-multiaccurate predictor to obtain a  $(\alpha, 2 \alpha - 1 - 2 \tau)$-weak agnostic learning for $\mC$ under the distribution $\mD$. In particular, any $(\mC, \tau)$-multiaccurate predictor $p$ for $\mD$ satisfies 
	\[ \cor_{\mD}(\y, 2 p(\x) - 1) \geq 2 \max_{c \in \mC} \cor_{\mD}(\y, c(\x)) - 1 - 2 \tau. \]
\end{theorem}
\begin{proof}
	Let $p$ be a $(\mC, \tau)$-multiaccurate predictor with respect to $\mD$. We have the following for any $c \in \mC$.
	\begin{align*}
		\cor_{\mD}(\y, c)  &= \cor_{\mD}(p^*, c) \\
											  &= \E_{\mD}[(2p^*(\x) - 1) c(\x)]  \\ 
											  &= \E_{\mD}[c(\x) (p^*(\x) - (1 - p^*(\x))) (p(\x) + (1 - p(\x)))]  \\ 
											  &= \E_{\mD}[c(\x) (p^*(\x) p(\x) - (1 - p^*(\x))(1 - p(\x))] + \E_{\mD}[c(\x) (p^*(\x) - p(\x))]  \\ 
											  &\leq \E_{\mD}[|c(\x)| |p^*(\x) p(\x) - (1 - p^*(\x))(1 -p(\x))|] + \tau\\ 
											  &\leq \E_{\mD}[p^*(\x)p(\x) + (1 - p^*(\x))(1 - p(\x))] +\tau.
	\end{align*}
	Above, in the penultimate inequality, we used the fact that $p$ satisfies $(\mC, \tau)$-multiaccuracy with respect to $\mD$, and in the final inequality, we use the fact that $|c(\x)| \leq 1$, and that 
	\[ |p^*(\x) p(\x) - (1 - p^*(\x))(1 - p(\x))| \leq p^*(\x) p(\x) + (1 - p^*(\x))( 1 - p(\x)), \]
	since both $p(\x), p^*(\x) \in [0, 1]$. 

	On the other hand, one can easily check the following identity for $a, b \in \R$:
    \[ (2a -1)(2b -1) = 2(ab + (1-a)(1- b)) -1.\]
    Using this we write
	\begin{align*}
		\cor_{\mD}(\y, 2 p - 1) &= \cor_{\mD}(p^*, 2p - 1) \\
											 &= \E_{\mD}[(2 p^*(\x) - 1) (2 p(\x) - 1)]  \\  
											 &= 2 \E_{\mD}[p^*(\x) p(\x) + (1 - p^*(\x))(1 - p(\x))] - 1.
	\end{align*}
	Combining the two inequalities, we get that,
	\begin{align*}
		\cor_{\mD}(\y, 2 p - 1) &\geq 2 \max_{c \in \mC} \cor_{\mD}(\y, c) - 1 - 2 \tau,
	\end{align*}
	which proves the claim.
\end{proof}

\subsection{Strong Agnostic Learning from Calibrated Multiaccuracy} \label{sec:MA+cal-to-WAL}

In this section, we show that provided the predictor $p$ that is $(\mC, \tau)$-multiaccurate is also well-calibrated, then $\sign(2p -1)$ in fact is a \emph{strong} agnostic learner. Interestingly, our proof decomposes the roles of \emph{multiaccuracy} and \emph{calibration} neatly and highlights the contribution that both these properties make in order to obtain a strong agnostic learning guarantee. We will assume in this section that the hypothesis class $C$ contains functions from $\X \to [-1, 1]$; as discussed in Section~\ref{sec:prelims-multigroup-fairness}, this is essentially without loss of any generality. 

Assume that $p$ is well-calibrated. It is intuitive that $p$ should do much better than random guessing, since when it predicts $v \in [0,1]$, the expected value of the label is indeed $v$. The only exception to this is if the predictor always predicts values $v$ close to $1/2$. We can formalize this as follows:

\begin{lemma} \label{lem:cal-to-al}
	Let $p$ be $\tau$-calibrated for some distribution $\mD$ over $\X \times \{0, 1\}$. Then
	\[ \cor_{\mD}(\y, \sign(2 p(\x)-1)) \geq 2\E\lt[\lt|p(\x) - \fr{2}\rt|\rt] - 2\tau.\]
\end{lemma}
\begin{proof}
    We can write
    \begin{align*}
		 \cor_{\mD}(\y, \sign(2p(\x)-1)) &=  \E_{\mD}[\sign(2p(\x) -1)(2\y- 1)]\\
														 &=  \E_{\mD}[\sign(2p(\x) -1)(2\E[\y|p(\x)]- 1)]\\
														 &= \E_{\mD}[\sign(2p(\x) -1)(2p(\x) - 1)] + 2\E_{\mD}[\sign(2p(x) -1)(\E[\y|p(\x)] -p(\x))]\\
														 &\geq  \E_{\mD}[|2p(\x) -1|] -2 \E[|\E[\y|p(\x)] -p(\x)|] \\
														 &\geq 2 \E_{\mD}\lt[ \lt|p(\x) - \frac{1}{2}  \rt| \rt] - 2\tau.
    \end{align*}
	 The last inequality follows from the definition of $\tau$-calibration.
\end{proof}

Next we show that if $\mC$ contains a hypothesis that correlates well with the labels, and if $p$ is $(\mC, \tau)$-multiaccurate, then $p$ cannot always predict values close to $1/2$. 

\begin{lemma} \label{lem:ma-to-al}
	If $p$ is $(\mC, \tau)$-multiaccurate for the distribution $\mD$, then for every $c \in \mC$, 
	\[ \E_{\mD}\lt[\lt|p(\x) - \fr{2}\rt|\rt] \geq \frac{\cor_{\mD}(\y, c(\x))}{2} - \tau.\]
\end{lemma}
\begin{proof}
    Fix any $c \in C$. Since $p$ is $(\mC,\tau)$-multiaccurate, we have 
    \begin{align}
		 \label{eq:c1}
		 \E_{\mD}[c(\x)(\y - p(\x))] &\leq \tau.
		 \intertext{We can also write}
		 \E_{\mD}\lt[c(\x)\lt(\y - \fr{2}\rt)\rt] &= \frac{\cor_{\mD}(\y, c(\x))}{2}.\label{eq:c2}
		 \intertext{Subtracting Equation \eqref{eq:c1} from \eqref{eq:c2} gives}
		 \E_{\mD}\lt[c(\x)\lt(p(\x) - \fr{2}\rt)\rt] &\geq \frac{\cor_{\mD}(\y, c(\x))}{2} - \tau. \nonumber
 	\end{align}
	The claim now follows by H\"older's inequality, since
	\begin{align*}
	 \E_{\mD}\lt[c(\x)\lt(p(\x) - \fr{2}\rt)\rt] &\leq \max_{\x} |c(\x)| \E\lt[\lt|p(\x) - \fr{2}\rt|\rt] \leq \E\lt[\lt|p(\x) - \fr{2}\rt|\rt],
    \end{align*}
    given that $|c(\x)| \leq 1$ for all $c \in \C$.
\end{proof}
The proof of Theorem~\ref{thm:calma-to-sal} below follows immediately from
Lemmas~\ref{lem:cal-to-al} (which applies to $\tau$-calibrated predictors) and~\ref{lem:ma-to-al} (which applies to $(\C, \tau)$-multiaccurate predictors). This shows provided that for
any $\tau > 0$ it is possible to construct a predictor $p$  that is $(\mC,
\tau)$-multiaccurate for $\mD$ such that $p$ is also $\tau$-calibrated, then
$\sign(2p -1)$ yields a strong agnostic learning guarantee for the class $\mC$
under $\mD$.
\begin{theorem} \label{thm:calma-to-sal}
	Let $\mC$ be a hypothesis class and $\tau > 0$. Suppose $p$ is a predictor that is $(\mC, \tau)$-multiaccurate and $\tau$-calibrated with respect to $\mD$. Then, $\sign(2 p - 1)$ satisfies,
	\[ 
    \cor(\y, \sign(2p(\x) - 1))] \geq \max_{c \in \mC} \cor(\y, c(\x)) - 4 \tau. 
    \]
\end{theorem}
\begin{proof}
    \begin{align*}
        \cor(\y, \sign(2p(\x) - 1)) &\geq 2\E\lt[\lt|p(\x) - \fr{2}\rt|\rt] - 2\tau & \text{$p$ is $\tau$-calibrated (Lemma~\ref{lem:cal-to-al})} \\
        &\geq 2 \cdot \Bigg( \frac{\cor_{\mD}(\y, c(\x))}{2} - \tau \Bigg) -2\tau & \text{$p$ is $(\C, \tau)$-multiaccurate (Lemma~\ref{lem:ma-to-al})} \\
        &\geq \cor_{\mD}(\y, c(\x)) - 4\tau.
    \end{align*}
\end{proof}

%% file: 4-hardcore.tex
\section{Hardcore Measures with Optimal Density}

In this section we consider a domain $\X$ endowed with a base distribution $\mD_\X$. We encourage the reader to think of $\X = \zo^n$ and $\mD_\X$ to be the uniform distribution. A \emph{measure} on $\X$ is a function $\mu: \X \rightarrow [0,1]$, which is not identically $0$. 
A measure $\mu$ defines a probability distribution $\bar{\mu}$\ by setting
\[
    \bar{\mu}(\x) = \frac{\Pr_{\mD_\X}[\x =x] \mu(\x)}{\E_{\x \sim \mD_\X}[\mu(\x)]}. 
\]
Given a measure $\mu: \X \rightarrow [0,1]$ and a base distribution $\mD_{\X}$ on $\X$, the \emph{density} of $\mu$ in $\mD_\X$ is given by
    \[
        \dns(\mu) = \E_{\x \sim \mD_{\X}}[\mu(\x)].
    \]
    We drop $\mD_\X$ from the density notation, but note that this is density with respect to the base distribution.
    Density measures the min-entropy of a distribution since
    \begin{align}
        \label{eq:min-ent} 
     \max_{x \in \X}\frac{\bar{\mu}(x)}{\Pr_{\mD_\X}[\x =x]} = \max_{x \in \X}\frac{\mu(x)}{\E_{\x \sim \mD_\X}[\mu(\x)]}  = \frac{\max_x \mu(x)}{\dns(\mu)} \leq \frac{1}{\dns(\mu)}
     \end{align}
When considering hardcore measures, we will restrict $\mC = \{c:\X \to \zo\}$ to be a class of Boolean functions, rather than bounded functions. We will assume that it contains the constant functions $\mathrm{0}, \mathrm{1}$ and is closed under complement. In complexity theory, given a class of distinguisher functions $\mC$, we are interested in studying how well the functions $c \in \C$ can approximate the outputs of a given boolean function $g$.
We quantify the \emph{hardness} of $g$ with respect to a class of functions $\mC$ and a distribution $\mD$ on $\mX$ as follows:
\begin{definition}[Hardness of a function]
    Given a class $\mC = \{c: \X \rightarrow \{0,1\}\}$, a function $g: \mX \rightarrow \{0,1\}$, a distribution $\mD_{\X}$ on $\mX$, and $\delta>0$, we say that $g$ is $(\mC, \delta)$-\emph{hard} on $\mD_{\X}$ if for all $c \in \mC$,
    \[
        \Pr_{\x \sim \mD_{\X}}[c(\x) = g(\x)] \leq 1-\delta.
    \]
\end{definition}
We say a distribution is hardcore when $\delta$ becomes arbitrarily close to $1/2$, since this means that it is hard to predict the values of $g$ much better than random guessing. A hardcore measure $\mu$ is one which induces a hardcore distribution $\bar{\mu}$. 

\begin{definition}[Relative complexity of a function and a function class {\cite[Definition 6]{jp14}}]\label{def:relcomplexity}
Let $\mC$ be a family of functions $f\colon \X \rightarrow [0,1]$. A function $h$ has \emph{complexity $(t, q)$ relative to $\mC$} if it can be computed by an oracle-aided circuit of size $t$ with $q$ oracle gates, where each oracle gate is instantiated with a function from~$\mC$.
We denote by $\mC_{t, q}$ the class of functions that have complexity at most $(t, q)$ relative to $\mC$. 
\end{definition}

Impagliazzo's hardcore lemma \cite{imp95} studies boolean functions $g$ that are somewhat hard to predict, and it shows that for these functions we can always find a hardcore measure $\mu$ (called a \emph{hardcore measure}) such that, under the distribution $\bar{\mu}$, the function $g$ is very hard to compute with respect to $\C$, one cannot improve much over random guessing. Formally:

\begin{theorem}[IHCL, \cite{imp95, hol05}]\label{thm:ihcl}
Let $\C = \{c: \X \rightarrow \{0,1\}\}$ be a family of functions, let $\mD_{\X}$ be a probability distribution over $\X$, and let $\epsilon, \delta > 0$. 
There exist $t = \mathrm{poly}(\log|\X|, 1/\epsilon, 1/\delta)$
and $q = \emph{poly}(1/\epsilon, 1/\delta)$ 
such that the following holds: If $g: \X \rightarrow \{0,1\}$ is $(\mC_{t, q}, \delta)$-hard on $\mD_{\X}$, then there is a measure $\mu$ satisfying:
\begin{itemize}
    \item {Hardness:}  $g$ is $(\mC, 1/2-\epsilon)$-hard on $\bar{\mu}$.
    \item {Optimal density:} $\dns(\mu)=2\delta$.
\end{itemize}
\end{theorem}
Impagliazzo's original proof gives density $\delta$ \cite{imp95}, the improvement to the optimal $2\delta$ is due to Holenstein \cite{hol05}.

We now define the set of balanced under calibration measures, which play a key role in our result. 
\begin{definition}\label{def:wmax-mumax}
    Define the weight function $\wmax:[0,1] \to [1,2]$ as $\wmax(p) = 1/\max(p, 1-p)$ and the weight function family
    \[ W = \{w: [0,1] \to [1,2] \ s.t. \ w(p) \in [1, \wmax(p)]\}.\]
    Given $g:\X \to \zo$ and $p: \X \rightarrow [0,1]$, we define the set $M(p, g)$ of {\em balanced under calibration} measures to be $\{\mu_w\}_{w \in W}$ where
    \[ \mu_w(x) = w(p(x)) \cdot |g(x) - p(x)| \ \text{for} \ w \in W. \]
    We define the the maximal {\em balanced under calibration} measure $\hmH \in M(p,g)$ as
    \[ \hmH(x) = \wmax(p(x)) \cdot |g(x) - p(x)| \]
    while the minimal {\em balanced under calibration} measure is $\ttvH(x) = |g(x)  - p(x)|$, corresponding to $w =1$.
\end{definition}

The condition $\mu(x) \in [0,1]$ requires the constraint $w(p(x)) \leq \wmax(p(x)) \leq 2$.
We also impose the condition $w(p) \geq 1$ since the goal is to find large/dense measures, so there is no point in choosing weights less than $1$. 
Intuitively, this ensures that, across the level sets $S_v$, the measure $\hmH$ assigns higher probability measure to level sets with $v$ value closer to $1/2$, which increases the hardness of $g$ when sampling according to the measure.
We note following chain of inequalities:
\[ \ttvH(x) \leq \mu(x) \leq \hmH(x) \leq 2\ttvH(x).\]

Some comments on the density and hardness of distributions in $M(p,g)$:
\begin{itemize}
\item As $w_\mu$ increases from $1$ to $\wmax(p)$, $\mu(x)$ and hence $\dns(\mu) = \E_{\mD_{\X}}[\mu(\x)]$ increases. However, it is less clear how the hardness of the distribution $\bar{\mu}$ changes, since this changes the balance of weight between the different level sets. Indeed, as we shall see, all these distributions are hard, under suitable assumptions on $p$ ($p$ being multicalibrated will suffice). 

\item  In contrast with $\ttvH$, $\hmH$ gives (nearly) twice as much weight to the level sets where $p(x) \approx 1/2$, which intuitively is the region where $g$ is hardest for the predictors in $\C$. So one might hope that (under suitable assumptions about the predictor $p$) $\hmH$ is indeed hardcore, and has density $2\delta$. 

\item For $p$ close to $\zo$, $\wmax(p) \approx 1$ and $\muMax$ and $\muTTV$ are almost equal. For distributions where $p$ takes such values with reasonable probability, if $\muMax$ is indeed $2\delta$ dense, it suggests that the density of $\muTTV$ is in fact greater than $\delta$.
\end{itemize}

\subsection{Weighted Multiaccuracy and Hardcore Measures}

We define the notion of weighted multiaccuracy, which applies to distributions on $\X \times \zo$. 

\begin{definition}[Weighted multiaccuracy]
Let $\mD$ be a distribution on $\X \times \zo$, let $\C = \{c: \X \rightarrow \zo\}$ be a class of functions, let $p: \X \rightarrow [0,1]$ be  predictor, and $w: [0,1] \rightarrow \mathbb{R}^{+}$ a weight function. We define the $(w, \mC)$-\emph{multiaccuracy error} of $p$ under the distribution $\mD$ as
\[ \MA_{\mD}(w, \mC, p) = \max_{c \in \mC} \Big| \E_{(\x, \y) \sim \mD}[c(\x)  w(p(\x)) (\y-p(\x))] \Big| .\]
We say that the predictor $p$ is $(w, \mC, \eps)$-\emph{multiaccurate} for $\mD$ if $\MA_{\mD}(w, \mC, p)  \leq \epsilon$.
\end{definition}

Given  a distribution $\mD_{\X}$ on $\X$ and $g:\X \to \zo$, define the distribution $\mD^g$ on $(\x, \y) \in \X \times \zo$ where $\x \sim \mD_{\X}$ and $\y = g(\x)$. We relate the hardness of predicting $g$ using functions from $\mC$ under $\bar{\mu}_w$ to the $w$-weighted multiaccuracy error of $p$ for the class $\mC$ under distribution $\mD^g$.
\begin{lemma}
\label{lem:ma-hardness}
    If $\MA_{\mD^g}(w, \mC, p) \leq \eps$, then $g$ is $\big(1/2 - 3\eps/(2\dns(\mu_w))\big)$-hard for $\mC$ under $\bar{\mu}_w$.
\end{lemma}

Before we prove Lemma~\ref{lem:ma-hardness}, we show some intermediate results. Let $\bar{\mu}$ denote a distribution on $\X$. 
Let $\eta= \E_{\bar{\mu}}[g(x)]$. 
The function class $\mC$ can trivially approximate $g$ to accuracy $\max(\eta, 1- \eta)$ under the distribution $\barmu$ by using a constant function. To characterize how much better $\mC$ can do, we define 
\begin{align*}
 \cor_{\barmu}(g, \mC) = \max_{c \in \mC} |\cor_{\barmu}(g, c)| = \max_{c \in \mC}\lt|  \E_{\x \sim \barmu}[(2g(\x) -1)c (\x)] \rt|.  
 \end{align*}
The next lemma shows that this captures the hardness of $g$ for the class $\mC$, under the distribution~$\mD$.

\begin{lemma}
\label{lem:adv-cor}
    We have 
    \begin{align}
    \label{eq:adv-cor}
    \min(\eta, 1 - \eta) + \cor_{\barmu}(g, \mC) \leq \max_{c \in \mC}\Pr_{\x \sim \bar{\mu}}[c(\x) = g(\x)] \leq \max(\eta, 1 - \eta) + \cor_{\barmu}(g, \mC). 
    \end{align}
\end{lemma}
\begin{proof}
    
    We can write
    \begin{align*}
        \Pr_{\barmu}[c(\x) = g(\x)] &= \E_{\barmu} \left[\fr{2}(1 + (2c(\x)  -1)(2g(\x) -1))\right]\\
        &= \E_{\barmu}[c(\x)(2g(\x)  -1) + (1 - g(\x))]\\
        &= \cor_{\barmu}(g, c) + 1 - \eta.
        \intertext{Recall that $\bar{c} = 1 - c$ and that we are assuming that $\mC$ is closed under complementation. So we have}
        \Pr_{\barmu}[\bar{c}(\x) = g(\x)] &= 1 - 
        \Pr_{\barmu}[c(\x) = g(\x)]\\
        & = \E_{\barmu}[g(\x)] - \E_{\barmu}[c(\x)(2g(\x)  -1)] ]\\
        &= \eta - \cor_{\barmu}(g, c).
    \end{align*}
  The upper bound of $\max(\eta, 1 - \eta) + |\cor_{\barmu}(g,c)|$  holds in both cases. The upper bound follows by maximizing over all $c \in \mC$. 
    
    To prove the lower bound, in order to predict $g$ with accuracy $\min(\eta, 1- \eta) + |\cor_{\barmu}(g, c)|$, we use $c$ if $\cor_{\barmu}(g,c) \geq 0$; else we use $\bar{c}$, which also lies in $\mC$ by assumption. Maximizing over $c \in \mC$ gives the claimed bound. 
\end{proof}

Equation \eqref{eq:adv-cor} gives tight bounds on the prediction advantage  of $\mC$ for balanced functions where $\eta \approx 1/2$. In general, it characterizes the advantage to within an additive $|1 - 2\eta|$, since
\[ \max(\eta, 1- \eta) - \min(\eta, 1- \eta) = | 1 - 2\eta|.\]
Lemma~\ref{lem:adv-cor} can be seen as a generalization of Yao's Lemma relating unpredictability with indistinguishability (see Lemma 2.14 in \cite{casacuberta2024complexity} for a formal statement).

 For balanced under calibration measures, we can relate correlation of $\mC$ with $g$ under the distribution $\bar{\mu}_w$ to the $(w, \mC)$-weighted multiaccuracy error for the labeled distribution $\mD^g$ on $(\x, \y)$ where $\x \sim \mD$ and $\y = g(\x)$. 

\begin{lemma}
\label{lem:wma-cor}
Consider a distribution $\mD_{\X}$ on $\X$,  a labeling function $g: \X \to \zo$, a predictor $p$, a weight function $w \in W$ and the corresponding measure $\mu_w \in M(p, g)$. Then
\[ \MA_{\mD^g}(w, \mC, p) = \cor_{\bar{\mu}_w}(\mC, g) \cdot \dns(\mu_w).\]
\end{lemma}
\begin{proof}
    For $y \in \zo$ and $t \in [0,1]$ we have the identity $|y- t|\cdot(2y-1) = y- t$.
    Applying this with $\y = g(\x)$ and $t = p(\x)$, we can write
    \begin{align}
    \label{eq:ma-cor}
        \E_{(\x,\y) \sim \mD^g}[c(\x)  w(p(\x)) (\y -p(\x))] &= \E_{\x \sim \mD_{\X}}[c(\x)  w(p(\x)) (g(\x) -p(\x))]\notag\\
        &= \E_{\x \sim \mD_{\X}}[c(\x)w(p(\x)) \cdot |g(\x) - p(\x)| \cdot (2g(\x) -1)]\notag\\
        &= \E_{\x \sim \mD_{\X}}[\mu_w(\x)c(\x)(2g(\x) -1)]\notag\\
        &= \frac{\E_{\x \sim \mD_{\X}}[\mu_w(\x)c(\x)(2g(\x) -1)]}{ \E_{\x \sim \mD_{\X}}[\mu_w(\x)]} \cdot \E_{\x \sim \mD_{\X}}[\mu_w(\x)]\notag\\
        &= \E_{\x \sim \bar{\mu}_w}[c(\x)(2g(\x) -1)] \cdot \dns(\mu_w)
    \end{align}
    Two steps above, we used the fact that $\mu_w(\x) = w(p(\x)) \cdot |g(\x) - p(\x)|$. 
    Taking the absolute value on either side and maximizing over all $c \in \mC$ gives the claimed result. 
\end{proof}
We use this to prove that weighted multiaccuracy of $p$ under $\mD^g$ implies $g$ is hard for $\mC$ under~$\bar{\mu}_w$. 

\begin{proof}[Proof of Lemma \ref{lem:ma-hardness}]
    We first show that multiaccuracy implies $\eta = \E_{\bar{\mu}_w}[g(x)]$ is close to $1/2$. Since $\mathbf{1} \in \mC$, we can set $c = \mathbf{1}$ in Equation \eqref{eq:ma-cor} to get
  
    \begin{align*} 
  \E_{\x \sim \mD_{\X}}[w(p(\x)) (g(\x) -p(\x))]|  = \E_{\x \sim \bar{\mu}_w}[(2g(\x) -1)] \cdot \dns(\mu) = (2\eta -1) \cdot \dns(\mu_w).
    \end{align*}
    Since this is bounded in absolute value by $\eps$, given that $\MA_{\mD^g}(w, \mC, p) \leq \eps$, we conclude that
    \[ \left|\eta - \fr{2}\right| \leq \frac{\eps}{2\dns(\mu_w)}. \]
    By Lemma \ref{lem:wma-cor}, we can use multiaccuracy to bound 
    \[ \cor_{\bar{\mu}_w}(\mC,g) \leq \frac{\eps}{\dns(\mu_w)}.\]
    We plug these bounds into Lemma \ref{lem:adv-cor} to get
    \begin{align*}
        \max_{c \in \mC}\Pr_{\x \sim \bar{\mu}_w}[c(\x) = g(\x)] &\leq \max(\eta, 1 - \eta) + \cor_{\bar{\mu}_{w}}(\mC, g)\\
         & \leq 
         \fr{2} + \frac{\epsilon}{2\dns(\mu_w)} + \cor_{\bar{\mu}_w}(\mC, g)\\
         & \leq \fr{2} + \frac{3\eps}{2\dns(\mu_w)}
    \end{align*}
    which implies the desired hardness.
\end{proof}

\subsection{Optimal Density Analysis Using Calibration}

\cite{TrevisanTV09} showed a lower bound on the density of $\ttvH$ assuming that $g$ that $g$ is $\delta$-hard for $\mC_{t,q}$, which is  expressive enough to compute $\ind{p(x) \geq \psi}$ for $\psi \in [0,1]$. Since $\ttvH$ is the minimal measure in $M(p,g)$, this implies the same lower bound for every measure in $M(p,g)$. We recall their argument below.
 
\begin{lemma}[\cite{TrevisanTV09}]
\label{lem:ttv-dense}
    Assuming $g$ is $\delta$-hard for $\mC_{t,q}$ under distribution $\mD_{\X}$, for every $\mu \in M(p,g)$, we have $\dns(\mu) \geq \delta$.
\end{lemma}
\begin{proof}
    We have $\dns(\muTTV) = \E[|g(\x) - p(\x)|]$. But $|g(\x) - p(\x)|$ is the error probability of the randomized function which outputs $1$ with probability $p(\x)$. By averaging, there exists a deterministic threshold $\psi \in [0,1]$ so that the deterministic function $\tilde{p}(x) = \ind{p(x) \geq \psi}$ does as well. But since $\tilde{p} \in \mC_{t,q}$, this implies that $\dns_{\mD}(\muTTV) \geq \delta$. 
    
    Every $\mu \in M(p,g)$ satisfies $\dns(\mu) \geq \dns(\ttvH)$, so the claim follows.
\end{proof}

Note that the analysis above uses very little about the predictor, other than an upper bound on its complexity. We will prove a stronger bound under the 
additional assumption that $p$ is a calibrated predictor for $g$ under $\mD^g$. This part of the argument will not rely on multiaccuracy, weighted or otherwise.

\begin{lemma}
\label{lem:cal-dense}
    Assume that the predictor $p$ is $\tau$-calibrated for $\mD^g$.
    Then we have
    \[ \lt|\dns(\mu_w) - 2\E_{\mD_{\X}}[w(p(\x))p(\x)(1 - p(\x))] \rt| \leq 2\tau.\]
\end{lemma}
\begin{proof}

    We will use the following identity for $v \in [0,1]$ and $y \in \zo$ which can be viewed as the multilinear expansion of $|y -v|$: 
    \[ |y - v| = y(1 - 2v) + v.\] 
    Using this with $y = g(x)$ and $v = p(x)$ we can write
    \begin{align*}
        \dns(\mu_w) &= \E_{\mD_{\X}}[\mu_w(\x)]\\   
            &= \E_{\mD_{\X}}[w(p(\x))\cdot|g(\x)- p(\x)|]\\
            &= \E_{\mD_{\X}}[w(p(\x))(g(\x)(1 - 2 p(\x)) + p(\x))]\\
            &=  \E_{\mD_{\X}}[w(p(\x))(p(\x)(1 - 2 p(\x)) + p(\x))] + \E_{\mD_{\X}}[w(p(\x))(1 - 2p(\x))(g(\x) - p(\x))]
    \end{align*}
    Hence we have
    \begin{align} 
        \label{eq:cal-condition}
    \lt|\dns(\mu_w) - 2\E_{\mD_{\X}}[w(p(\x))p(\x)(1 - p(\x))] \rt| = \lt|\E_{\mD_{\X}}[w(p(\x))(1 - 2p(\x))(g(\x) - p(\x))]\rt|.
    \end{align}
     Since $w(p) \in [1,2]$ and $1 -2p \in [-1,1]$, we have $|w(p)(1 - 2p)| \leq 2$. 
    Since $p$ is $\tau$-calibrated, by Equation \eqref{eq:alt-ece} applied to $v(p) = w(p)(1 -2p)/2$, the predictor $p$ satisfies the bound
    \begin{align*}
        \left|\E_{\mD_{\X}}[w(p(\x))(1 - 2p(\x))(g(\x) - p(\x))]\right| \leq 2\tau.
    \end{align*}
    Plugging this into Equation \eqref{eq:cal-condition} gives the claimed bound. 
\end{proof}

\begin{lemma}
\label{lem:hard-dense}
    Assume that $g$ is $\delta$-hard for $\mC_{t,q}$ under $\mD_\X$, and $p$ is $\tau$-calibrated for $\mD^g$. Then
    \[ \E_{\mD_{\X}}[\min(p(\x), 1 - p(\x))] \geq \delta - \tau.\]
\end{lemma}
\begin{proof}
    Consider the function $f(x) = \ind{p(x) \geq 1/2}$. This function is computable in $\mC_{t,q}$.
    The function $f$ errs when $g(x)  =1$ for points where $p(x) < 1/2$ and whenever $g(x) = 0$ otherwise. So we can write
    \begin{align*}
        \delta &\leq \Pr_{\mD_{\X}}[f(\x) \neq g(\x)]\\
        &= \Pr_{\mD_{\X}}[g(\x) =1 \wedge p(\x) < 1/2] + \Pr_{\mD_{\X}}[g(\x) = 0 \wedge p(\x) \geq 1/2] \\
        &= \E_{\mD^g}[\E[\y|p(\x)] \cdot \ind{p(\x) < 1/2} + (1 - \E[\y|p(\x)]) \cdot \ind{p(\x) \geq 1/2}]\\
        &= \E_{\mD^g}[(p(\x) + \E[\y|p(\x)] - p(\x))\cdot\ind{p(\x) < 1/2} + (1 - p(\x) + p(\x) - \E_{\mD^g}[\y|p(\x)])\cdot \ind{p(\x) \geq 1/2}]\\
        &\leq \E_{\mD^g}[(p(\x) \cdot \ind{p(\x) < 1/2} + (1 - p(\x))\cdot\ind{p(\x) \geq 1/2})] + \E_{\mD^g}[|p(\x) - \E[\y|p(\x)]|]\\
        &= \E_{\mD_{\X}}[\min(p(\x), 1 - p(\x)] + \ECE(p, \mD^g),
    \end{align*}
    where we use the identity that
    \[ \min(v, 1- v) = v \cdot \ind{v <1/2} + (1 - v) \cdot \ind{v \geq 1/2}. \]
    The claim now follows since $\ECE(p, \mD^g) \leq \tau$.
\end{proof}

Often, the assumption of bounded $\ECE$ can be replaced with the weaker assumption that $\E[v(p)(y - p)]$ is bounded for specific functions $v$. 
In the proof of Lemma \ref{lem:cal-dense}, we need to bound the calibration error for $v(p) = w(p)(1 -2p)/2$, while in the proof of Lemma \ref{lem:hard-dense}, we only needed to bound the error for $v(p) = 2\cdot \ind{p \geq 1/2} - 1$.

As a corollary, we get the following density  bound for $\ttvH$, which is always at least $\delta$, and is often better.
\begin{corollary}
Assume that $p$ is $\tau$-calibrated for $\mD^g$. Then 
    \[ \dns_{\mD_{\X}}(\ttvH) \geq \E_{\mD_{\X}}[2p(\x)(1 - p(\x))] - 2\tau. \]
\end{corollary}
Observe that 
 \[ \min(p(\x), 1 - p(\x)) \leq 2p(\x)(1 - p(\x)) \leq 2\min(p(\x), 1 - p(\x)), \]
 so by Lemma \ref{lem:hard-dense} we have $\E_{\mD_{\X}}[2p(\x)(1 - p(\x))] \in [\delta, 2\delta]$. 
 
\subsection{Hardcore Measures with Optimal Density}
\label{sec:opt}

We now put together the ingredients to prove the existence of hardcore sets with optimal density.

\begin{theorem}
\label{thm:2delta}
    For $\eps, \delta, \tau > 0$ where $\tau \leq \delta/2$, assume that $p$ is $(\wmax, \mC, \eps\delta)$-multiaccurate for $\mD^g$, and it is $\tau/4$-calibrated. 
    Then the measure $\hmH$ is $2\delta - \tau$ dense in $\mD_{\X}$, and
    $g$ is $(1/2 - \eps)$-hard for $\mC$ under the distribution $\bar{\mu}_{\mathsf{Max}}$.
\end{theorem}
\begin{proof}
    By Lemma \ref{lem:cal-dense} applied with $\wmax$, we have
    \begin{align*}
         \dns(\muMax) &\geq 2\E_{\mD_{\X}}\left[\frac{p(\x)(1 - p(\x))}{\max(p(\x), 1 - p(\x))}\right] - \tau/2\\
         &= 2\E_{\mD_{\X}}[\min(p(\x), 1 - p(\x))] - \tau/2\\
         & \geq 2\delta - \tau/2 - \tau/2\\
         &\geq 2\delta - \tau.
    \end{align*}
    where we use Lemma  \ref{lem:hard-dense}. 
    
    We can lower bound how hard $g$ is for $\mC$ using Lemma \ref{lem:ma-hardness} by
    \[ \fr{2} - \frac{3\eps\delta}{2(2\delta - \tau)} \geq \fr{2} - \eps \]
    for $\tau < 0.5 \delta$.
\end{proof}

Next we consider the algorithmic complexity of obtaining a predictor that satisfies both calibration and $(\wmax, \mC)$-weighted multiaccuracy, assuming an oracle for $(\alpha, \beta)$ weak agnostic learning for $\mC$. Following previous work, we will consider the number of invocations to the weak agnostic learner. The key point is that since the only care about a single weight function $\wmax$ which is bounded by $2$,
the complexity of the weighted versions of multiaccuracy and calibrated multiaccuracy are the same as their unweighted counterparts. The work of \cite{lossOI} which introduced calibrated multiaccuracy shows that its complexity is comparable to that of multiaccuracy, and much lower than multicalibration \cite[Section 7]{lossOI}. This relies on ideas that are by now standard in the multicalibration literature \cite{hkrr2018, GopalanKSZ22, lossOI}, so we will provide a sketch.

Consider the problem of learning a predictor with $(\wmax, \mC, \alpha)$ multiaccuracy, given access to an $(\alpha, \beta)$-weak learner for $\mC$. This can be done using the weighted multicalibration algorithm in \cite[Section 5]{GopalanKSZ22}. The main difference is that we have a single weight function $\wmax$ rather than a set of weights, we consider binary rather than $l$-class classification, and their analysis assumes a weak learner where $\beta = \alpha/2$. We sketch the analysis below. Assume there exists $c \in \mC$ so that
\[ \E[c(\x)\wmax(p(\x))(\y - p(\x))] \geq \alpha.\]
If we let $\z = \wmax(p(\x))(\y - p(\x))$, then since $\y \in \{0, 1\}$ and $p(\x) \in [0, 1]$, $|y - p(\x)| \leq \max\{p(\x), 1 - p(\x)\}$ and so $\z \in [-1,1]$, and we know $\E[c(\x)\z] \geq \alpha$.  Hence by running our weak learner, we get $h$ such that $\E[h(\x)\z] \geq \beta$. Updating to $p'(\x) = p(\x) + \kappa\beta w(p(\x)) h(\x)$ for a suitable constant $\kappa$ reduces the squared loss by $\Omega(\beta^2)$. This implies we will converge in $O(1/\beta^2)$ iterations to a $(\wmax, \mC, \alpha)$-multiaccurate predictor. Essentially the same analysis works for all $w \in W$, and the case when $w =1$ is standard multiaccuracy. 

To achieve calibration in addition to weighted multiaccuracy, we use the analysis of the calibrated multiaccuracy algorithm from \cite[Section 7]{lossOI}. The key observation there is that if the calibration error is large, then recalibrating the predictor reduces the squared loss. Hence we can interleave recalibration steps with the calls to the weak learner, all the while ensuring that the squared loss decreases, and the number of calls to weak learner stays bounded by $O(1/\beta^2)$. 

Lastly, we use Theorem~\ref{thm:2delta} to derive IHCL with optimal hardcore density $2\delta$. 
We use the calibrated multiaccuracy algorithm from \cite[Section 7]{lossOI} to bound the complexity $\mC_{t,q}$ of the class on which the input function $g$ is assumed to be $\delta$-hard. 

\begin{theorem}\label{thm:our-ihcl-2delta}
    Let $\C$ be a family of functions $c: \X \rightarrow \{0, 1\}$, let $\mD_{\X}$ be a probability distribution over $\X$, and let $\epsilon, \delta, \tau > 0$. 
    There exist $t=O((1/(\epsilon^2 \delta^2) + 1/\tau) \cdot \log(|\X|/\epsilon))$, $q=O(1/(\epsilon^2 \delta^2))$ 
    such that the following holds: If $g: \X \rightarrow \{0,1\}$ is $(\C_{t, q}, \delta)$-hard on $\mD_{\X}$, then there is a measure $\mu$ satisfying:
    \begin{itemize}
        \item Hardness: $g$ is $(\C, 1/2-\epsilon)$-hard on $\bar{\mu}$.
        \item Optimal density: $\dns(\mu) = 2\delta-\tau$.
    \end{itemize}
\end{theorem}

\begin{proof}
    We call the $\mathsf{calMA}$ algorithm from \cite{lossOI} with $\C, \mD^g$, weight function $\wmax$, desired multiaccuracy error $\epsilon \delta$, and desired calibrated error $\tau/4$, and obtain a $(\wmax, \C, \epsilon \delta)$-multiaccurate and $(\tau/4)$-calibrated predictor $p$ for $\mD^g$.
    Using this predictor $p$, the input function $g$, and the weight function $\wmax$, we construct the measure $\hmH$ as per Definition~\ref{def:wmax-mumax}; namely, $\hmH(x) = \wmax(p(x)) \cdot |g(x) - p(x)|$.
    By Theorem~\ref{thm:2delta}, this measure $\hmH$ achieves the claimed hardness and optimal density stated in Theorem~\ref{thm:our-ihcl-2delta}.
    For the $\mathsf{calMA}$ algorithm, 
    we assume access to a proper $(\epsilon \delta, \epsilon \delta)$-weak agnostic learner for $\C$ under $\mD_{\X}$; this satisfy the parameter requirements in \cite[Alg. 2]{lossOI}.\footnote{In \cite{lossOI}, the total number of calls to an $(\alpha, \beta)$-weak agnostic learner is analyzed. Here we are not concerned with the efficiency of the learner so we can assume a strong learner giving $\alpha=\beta$.}
    
    By Theorem~7.7 in \cite{lossOI}, the total number of calls to the learner during the execution of the $\mathsf{calMA}$ algorithm is bounded by $O(1/(\epsilon \delta)^2)$, which yields the claimed $q$ parameter.
    We remark that in their algorithm, the recalibration step does not call the weak agnostic learner, which is why the calibration error $\tau/4$ is not part of the $q$ parameter.
    As for the total number of gates in the circuit (i.e., the $t$ parameter), each iteration of the multiaccuracy step in the $\mathsf{calMA}$ algorithm requires a scalar multiplication, a finite-precision addition, and a projection onto the interval $[0,1]$.
    As argued in \cite[Appx. A]{casacuberta2024complexity}, this can be handled by $O(\log(1/\epsilon)+\log(1/|\X|)$ gates, where we perform computations up to a fixed precision of $\Theta(\epsilon)$ and $\log(1/|\X|)$ corresponds to the bit-length of the elements in $\X$.
    Lastly, the calibration-based steps in the $\mathsf{calMA}$ algorithm do not require further calls to the learner, but they require computing all of the values of the predictor $p$ (\cite[Lemma 7.4]{lossOI}).
    We discretize $p$ with parameter $\tau$, so that the interval $[0,1]$ is partitioned into $O(1/\tau)$ intervals.
    To compute the $O(1/\tau)$ values of the predictor $p$, we can use a lookup table of size $O(1/\tau)$; such a lookup function can be computed by a circuit of size $O(1/\tau)$ \cite[\S 4]{barak2022introduction}.
\end{proof}

\paragraph{The number of oracle gates in \cite{casacuberta2024complexity} and other optimal hardcore set constructions.}
\cite{casacuberta2024complexity} uses the multicalibration algorithm to obtain a stronger and more general version of Impagliazzo's hardcore lemma, IHCL$++$.
From it,  they recover the IHCL theorem with  density $2\delta$.
However, because this recovery is done through the multicalibration algorithm, the $(t, q)$ parameters for the class $\C_{t, q}$, although polynomial in $1/\epsilon, 1/\delta$,  are much larger than the ones we have obtained in Theorem~\ref{thm:our-ihcl-2delta}.
In particular, using the multicalibration algorithm presented in \cite{omni}, we require $q = O(1/(\epsilon^6 \delta^6))$ calls to the weak agnostic learner \cite{lossOI}. When we recover the IHCL theorem from IHCL$++$, the parameters further increase to $q = 1/(\epsilon^2 \delta)^6$ oracle gates.\footnote{This is because we can only glue together the pieces in the partition that are large enough and balanced enough; see \cite[\S 3.3]{casacuberta2024complexity} for details.}

Moreover, the dependence on $\epsilon$ in the number of oracle calls $q$ that we make to the weak agnostic learner in our Theorem~\ref{thm:our-ihcl-2delta} is optimal.
All boosting-based proofs of the hardcore lemma, starting with one of Impagliazzo's original proofs and later generalized by Kilvans and Servedio to modularly use various boosting algorithms to obtain hardcore measures, need to call a weak learner $O(1/\epsilon^2)$ many times as to obtain constant density measure on which $g$ is $(\C, 1/2-\epsilon)$-hard \cite{ks03}.

In particular, Impagliazzo's boosting-based proof requires $q=O(1/(\epsilon^2 \delta^2))$ oracle calls \cite[Thm. 1]{imp95}, and all subsequent hardcore constructions retain the dependence on $1/\epsilon^2$ in $q$ \cite{ks03, hol05}, including the uniform versions of the hardcore lemma \cite{trevisan2003list, barak2009uniform, vadhan2013uniform}.

The size loss of $O(1/\epsilon^2)$ between the class $\C_{t, q}$ (on which we assume the input function $g$ is mildly hard) and the class $\C$ (on which $g$ is strongly hard) is inevitable in all boosting-based proofs of IHCL \cite{lu2011complexity}, Blanc, Hayderi, Koch, and Tan recently showed that this circuit size loss is necessary and that these parameters are in fact optimal and inherent to the hardcore lemma, regardless of proof strategy \cite{blanc2024sample}.
The dependence on $\delta$ can be improved: Nisan's min-max proof in Impagliazzo's original paper yields $q=O(1/\epsilon^2 \log(1/(\epsilon \delta))$ \cite[Lemma 2]{imp95}, and it was further improved to $q=O(1/\epsilon^2 \log(1/\delta))$ by Kilvans and Servedio  who showed an elegant connection to boosting \cite[Thm. 25]{ks03}. But these proofs do not give the optimal density.

\paragraph{Deriving IHCL from IHCL$++$.} The IHCL$++$ theorem of \cite{casacuberta2024complexity} partitions the space $\X$ into level sets $S_v$. Within each level set, they construct a hardcore $\mathcal{H}_v$ distribution of density $2\min(v, 1- v)$. We also know that $\E[\min(v, 1 -v)] \geq \delta$. To derive the IHCL with density $2\delta$, we need to glue these  distributions together to get a measure of density $2\delta$. 
In their paper, \cite{casacuberta2024complexity} suggest the following gluing: sample $S_v$ according to $\mD_\X$ and then sample within $S_v$ from $\mathcal{H}_v$. This gluing as stated might not give a distribution with bounded density. The reason being that when we take convex combinations of distributions over disjoint regions, the  min-entropy of the combination is the {\em maximum} of the min-entropies, not the average.
A simple fix is to go from hardcore distributions to sets, and then take the union of sets. 

We describe a different fix, which again leads to the measure $\hmH$. 
From among all the measures that correspond to the distribution $P_v$,  we pick the densest one, which happens to be $|y -v|/\max(v, 1-v) = \hmH(v)$. This gives density $2\min(v, 1-v)$ on $S_v$.
Intuitively, this gluing takes into account the value of $v$ (specifically, its closeness to $1/2$) into account, so that we give more weight to the level sets $S_v$ where $g$ is hardest to predict.
Gluing densities is simple, since we define the functions independently on each $S_v$. The overall density is indeed the average of the densities on each piece, so it is $2\delta$. 

\subsection{Density of $\ttvH$ vs $\hmH$}
We conclude this section  with an example illustrating the density difference between $\ttvH$ and $\hmH$.
Consider a domain $\X$ and a probability distribution $\mD$ on $\X$. 
We define a predictor $p: \X \rightarrow [0,1]$ that predicts only two values on $\X$: $p(x) = 1/2$ and $p(x) = \eta$ for some small value $0 < \eta < 1/2$. 
Let $S_{1/2} = \{x : p(x)=1/2\}$ and $S_{\eta} = \{x : p(x) = \eta\}$; their sizes are such that $\Pr_{\mD_{\X}}[\x \in S_{1/2}] = 2\delta$ and $\Pr_{\mD_{\X}}[\x \in S_{1/2}] = 1-2\delta$, where $\delta>0$ and $\eta \ll \delta$.
Moreover, $p$ is perfectly calibrated for $\mD^g$, where $g$ is the boolean input function.
Given these $p, g$, and $\mD$, we compare the densities of the measures $\ttvH$ and $\hmH$.
Figure~\ref{fig:measures} summarizes the construction (drawn for example values $\eta=0.05, \delta=0.2$). 

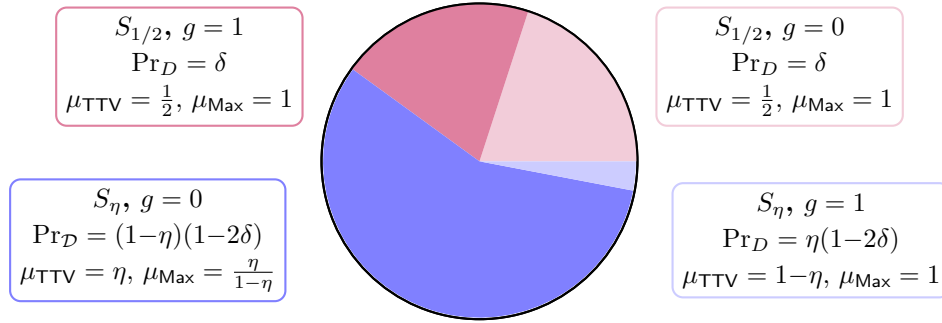
\begin{figure}[H]
\hspace{3cm}
\begin{center}
\begin{tikzpicture}[scale=2.1]

\draw[thick] (0,0) circle (1);

\fill[purple!20] (0,0) -- (0:1) arc (0:72:1) -- cycle;               
\fill[purple!50] (0,0) -- (72:1) arc (72:144:1) -- cycle;            
\fill[blue!50] (0,0) -- (144:0.99) arc (144:349.2:0.99) -- cycle;    
\fill[blue!20] (0,0) -- (349.2:0.99) arc (349.2:360:0.99) -- cycle;  

\draw[thick] (0,0) circle (1);


\node[draw=blue!50, line width=0.8pt, fill=white, rounded corners, align=center] at (-2.1,-0.5) {
  \small \textbf{$S_\eta$, $g=0$} \\
  \small  $\Pr_{\mathcal{D}} = (1{-}\eta)(1{-}2\delta)$ \\
  \small $\mu_{\mathsf{TTV}} = \eta$, $\mu_{\mathsf{Max}} =  \frac{\eta}{1{-}\eta}$
};

\node[draw=blue!20, line width=0.8pt, fill=white, rounded corners, align=center] at (2.1,-0.5) {
  \small  \textbf{$S_\eta$, $g=1$} \\
  \small  $\Pr_D = \eta(1{-}2\delta)$ \\
  \small  $\mu_{\mathsf{TTV}} = 1{-}\eta$, $\mu_{\mathsf{Max}} = 1$
};

\node[draw=purple!20, line width=0.8pt, fill=white, rounded corners, align=center] at (1.9,0.6) {
  \small  \textbf{$S_{1/2}$, $g=0$} \\
  \small  $\Pr_D = \delta$ \\
  \small  $\mu_{\mathsf{TTV}} = \frac{1}{2}$, $\mu_{\mathsf{Max}} = 1$
};

\node[draw=purple!50, line width=0.8pt, fill=white, rounded corners, align=center] at (-1.9,0.6) {
  \small  \textbf{$S_{1/2}$, $g=1$} \\
  \small $\Pr_D = \delta$ \\
  \small  $\mu_{\mathsf{TTV}} = \frac{1}{2}$, $\mu_{\mathsf{Max}} = 1$
};

\end{tikzpicture}
\end{center}
\caption{Weights assigned by $\mu_{\mathsf{TTV}}$ and $\mu_{\mathsf{Max}}$ on each region of $\X$. The two measures nearly agree on the {\em easy} blue region, but $\hmH$ is twice $\ttvH$ on the {\em hard} red region.}
    \label{fig:measures}
\end{figure}

\begin{itemize}
    \item $\mathbf{\mu_{\mathsf{TTV}}}:$ In $S_{1/2}$, $p = 1-p = 1/2$, and so $\ttvH$ assigns weight $1/2$ to all $\x \in S_{1/2}$. 
    By the calibration guarantee on $p$, within $S_{1/2}$ half the points have $g$-value equal to 0 and half the points have $g$-value equal to 1.
    In $S_{\eta}$, $\ttvH$ gives weight $p = \eta$ to the points where $g(x) =0$ and weight $1-p = 1-\eta$ to the points where $g(x)=1$.
    By the calibration guarantee on $p$, within $S_{\eta}$ a $(1-\eta)$-fraction of points have $g$-value equal to 0, and a $\eta$-fraction of points have $g$-value equal to 1.
    Therefore,
    \begin{align*}
        \dns(\ttvH) &= \E_{x \sim \mD_{\X}}[\ttvH(x)]\\
        &= \dfrac{1}{2} \cdot 2\delta + \eta \cdot (1-\eta) (1-2\delta) + (1-\eta) \cdot \eta (1-2\delta)\\ 
        &= \delta + 2\eta(1-\eta)(1-2\delta)\\
        &= \delta + O(\eta).
    \end{align*}
    \item $\mathbf{\mu_{\mathsf{Max}}}:$ In $S_{1/2}$, $p = 1 - p$, so $\hmH$ assigns weight 1 to all $\x \in S_{1/2}$.
    In $S_{\eta}$, $\hmH$ gives weight $\eta(1-\eta)$ to the points where $g(x)=0$ weight $1$ to the points where $g(x)=1$.
    Again using the calibration guarantee on $p$ within each of $S_{1/2}$ and $S_{\eta}$ it follows that
    \begin{align*}
        \dns(\hmH) &= \E_{x \sim \mD_{\X}}[\hmH(x)]\\
        &= 1 \cdot 2\delta + \dfrac{\eta}{1-\eta} \cdot (1-\eta)(1-2\delta) + 1 \cdot \eta(1-2\delta)\\
        &= 2\delta + 2\eta(1-2\delta)\\
        &= 2\delta + O(\eta).
    \end{align*}
\end{itemize}

Given that $\eta \ll \delta$, it follows that $\hmH$ is nearly twice as dense as $\ttvH$.\footnote{Strictly speaking, to compare the min-entropies, we should divide $\dns(\ttvH)$ by $\max_x \ttvH(x) = 1- \eta$ following Equation \eqref{eq:min-ent}. Since $\eta \ll \delta$, this can be absorbed in the $O(\eta)$ term. But this is why we need $\eta > 0$ in our construction: it ensures $\max_x \ttvH(x) \approx 1$, so that the min-entropy is the reciprocal of the density. If we take $\eta =0$, both $\ttvH$ and $\hmH$ will in fact give the same distribution.}

%% file: 6-meek-wal.tex
\section{Restricted Weak Agnostic Learning and Multiacccuracy}\label{sec:code}

Theorem~\ref{thm:MA-to-meek} proves a reduction to multiaccuracy from $(\alpha, \beta)$-weak agnostic learning for $\alpha > 1/2$. This motivates the study of weak learning in the setting where $\alpha$ is bounded away from $0$. In this section, we investigate the notion of \emph{restricted weak agnostic learning}, i.e., $(\alpha, \beta)$-weak agnostic learning when $\alpha$ is bounded away from $0$, and study what learning and fairness primitives can be derived from such \emph{restricted weak agnostic learners}. 

We begin by the observation that given any (restricted) $(\alpha, \beta)$-weak agnostic learner for $\C$, can we efficiently construct a $(\C, \alpha+\eps$)-multiaccurate predictor by the original multiaccuracy algorithm \cite{hkrr2018}. The following result is essentially proved in their paper and follows from the observation that if there is a multiaccuracy violation, then the weak agnostic learner can be used to provide a gradient update that reduces the potential function.

\begin{theorem}[\cite{hkrr2018}]\label{claim:wal-ma-standard}
    For any concept class $\C$, distribution $\mD$ and parameters $\alpha, \beta$, given access to a $(\alpha, \beta)$-weak agnostic learner for $\C$ for distributions with marginal $\mD_\X$, one can produce a $(\C, \alpha + \eps)$-multiaccurate predictor for $\mD$ using $O(1/\beta^2)$ oracle calls to the weak learner.
\end{theorem}

The remainder of the section is devoted to proving Theorem~\ref{thm:noise-hard} below, which shows that (a) in general, $(\alpha^\prime, \beta)$-weak agnostic learning cannot be reduced to $(\alpha, \beta)$-agnostic learning for $\alpha > \alpha^\prime$, (b) in general, we cannot obtain a predictor that is $(\mC, \alpha/2 - \eps)$-multiaccurate given access to an $(\alpha, \beta)$-weak agnostic learner. The latter result has a gap of a factor of two when viewed in conjunction with Theorem~\ref{claim:wal-ma-standard}. While closing this is an interesting question, we observe that this is not possible for all values of $\alpha$, e.g. when $\alpha = 1$, i.e., we only have \emph{realizable} weak learning, it is still always possible to achieve $(\mC, 1/2)$-multiaccuracy, by using the constant predictor $p(\x) = 1/2$.

Note that the task of $(\alpha, \beta)$ weak learning gets harder as 
\begin{itemize}
    \item $\alpha$ decreases,  since this corresponds to detecting lower correlation.
    \item $\beta$ increases, since this corresponds to learning a better hypothesis.
\end{itemize}
The agnostic boosting results of \cite{fel09, kk09} tell us that increasing $\beta$ does not necessarily result in harder learning problems. They show that an $(\alpha,\beta)$ weak learner implies a $(\gamma, \gamma - \alpha)$ weak learner for all $\gamma > \alpha$. This result is typically interpreted to say that if weak agnostic learning is possible for $\alpha$ approaching $0$, then strong agnostic learning is possible. 

It is natural to ask if there might be a similar boosting statement for $\alpha$: that having an $(\alpha, \beta)$ weak learner might imply an $(\alpha', \beta')$ weak learner for some values of $\alpha' < \alpha$. We will show that such a result is unlikely, assuming the existence of standard cryptographic primitives. 

For any $\alpha \in (0,1)$ and $\eps > 0$, we will show that there exist classes that are  agnostically learnable in polynomial time given correlation $\alpha + \eps$ with the labels, but are hard to learn in polynomial time if the correlation is bounded by $\alpha$. Further, in the former case, the learner only requires random samples and is a proper learner, but for the latter case, the weak learner is allowed query access, and is allowed to be improper.  This result shows that weak learning indeed becomes easier as $\alpha$ increases, supporting our claim that multiaccuracy might be weaker as a learning primitive than weak agnostic learning.

Our construction of classes that are hard to learn assumes the  existence of pseudorandom families of functions, and on the construction of binary error-correcting codes that can be (list)decoded from error rates approaching $1/2$ \cite{GuruswamiS00}. Decoding from error rate $1/2$ is information-theoretically impossible, since a random string would be at distance $1/2$ from a codeword. The gap between distance $1/2$ and $1/2 -\eps$ can be viewed as a gap of $0$ versus $2\eps$ in the correlation between the received word and codeword. Our result uses pseudorandom functions to translate this to a computational gap in the learnability of a suitably defined concept class with correlation  $\alpha$ versus $\alpha + \eps$. Our construction is inspired by a construction of Feldman \cite{Feldman08} which he uses to show that membership queries add power to agnostic learning. 

\begin{assumption}
\label{claim:prf}
    There exists a family of pseudorandom functions, denoted $\mF = \{f_r:\zo^n \to \zo~|~{r \in \zo^n}\}$, such that
    \begin{itemize}
        \item $f_r$ can be computed in time $\poly(n)$
        given $r$ as input.
        \item
        For $r \leftarrow \zo^n$ chosen uniformly at random, any (potentially randomized) algorithm that runs in time $T(n)$ and has query access to $f_r:\zo^n \to \zo$ cannot distinguish it from a truly random function with advantage greater than $1/T(n)$ over the uniform distribution on $x \in \zo^n$.
    \end{itemize}
\end{assumption}

Such functions are implied by the existence of one-way functions \cite{GoldreichGM86}. The time bound $T(n)$ of the distinguishers depend on the hardness of the one-way function. The most standard  assumptions would imply there are no distinguishers with $T(n) = n^c$, for any constant $c$. Often though, stronger assumptions are used where $T(n)$ is super-polynomial or even exponential.

The other key ingredient in our construction is binary error-correcting codes that can be list-decoded from high error rates in the worst case error model. 

\begin{definition}
    A code $\ecc:\zo^n \to \zo^m$ is said to be \emph{efficiently list decodable} at radius $\delta$ if given a received word $w \in \zo^m$ such that there exists $x \in \zo^n$ satisfying $d(w, \ecc(x)) \leq \delta m$, there is an algorithm that runs in time $\poly(n)$ and returns a list containing $x$ with probability at least $0.9$.  
\end{definition}

We will use constructions of binary error correcting codes that are efficiently list-decodable from error rate $1/2 - \eps$. Such codes were first shown to exist by \cite{GuruswamiS00}. We will use their list-decoder for Reed-Solomon concatenated with Hadamard codes.
\begin{lemma}\cite[Corollary 2]{GuruswamiS00}
\label{lem:gs}
    Let $\eps > 0$. For infinitely many $n$, there exists a code $\rsh:\zo^n \to \zo^m$ where $m = O(n^2/\eps^4)$ which is list-decodable at radius $1/2 - \eps$ in time $\poly(n, 1/\eps)$.%
	 \footnote{We ignore the question of what precise values of $(n, m)$ the code is defined for, but the underlying Reed-Solomon concatenated with Hadamard code family is defined for a fairly dense set of values.}
\end{lemma}
 While there have been subsequent improvements, our results only need $m = \poly(n, 1/\eps)$, which this construction guarantees.

\begin{theorem}
\label{thm:noise-hard}
For every $\alpha \in (0,1)$, there exists an instance space $\X = \cup_n \X_n$, distributions $(\mD_{\X_n})_{n \geq 1}$ where $D_{\X_n}$ is a distribution over $\X_n$, and a concept class $\mC = \cup_n \mC_n$, with $\mC_n$ containing concepts over $\X_n$, such that:
\begin{enumerate}[label=(\roman*)]
	\item There exists a learning algorithm $L$, such that for every $0 < \eps < 1 - \alpha$, $n \geq 1$, $\mC_n$ is $(\alpha +\eps, \alpha)$-weak agnostically learnable by $L$ in time $\poly(n, 1/\eps)$ for any distribution $\mD$ with marginal $\mD_{\X_n}$. The learner is proper and only requires only random examples.\label{thm:noise-hard-easy}
	\item Under Assumption~\ref{claim:prf}, there does not exist any learning algorithm $L$ that runs in time $T(n)/\poly(n)$ that is an $(\alpha, \beta)$-weak agnostic learner for $\mC_n$ for any $\beta \geq 2\alpha/T(n)$ for distributions with marginal $\mD_{\X_n}$ over $\X_n$, and succeeds with probability at least $1/T(n)$, even with access to membership queries. \label{thm:noise-hard-hardness1}
	\item Under Assumption~\ref{claim:prf}, for any $0 < \eps < \min\{\alpha, 1 - \alpha\}$, there does not exist any algorithm $\mA$ that runs in time $T(n)/\poly(n, 1/\eps)$ and outputs with probability at least $1/T(n)$, a predictor $p : \X_n \rightarrow [0, 1]$ that is $(\C_n, \alpha/2 - \epsilon)$-multiaccurate for distributions with marginal $\mD_{\X_n}$ over $\X_n$, even with access to membership queries. \label{thm:noise-hard-hardness2}
    \end{enumerate}
\end{theorem}

\begin{proof}
	The basic setup for the three parts of the theorem are the same which we
	first describe.  For simplicity we will drop the subscript $n$, keeping the
	dependence on $n$ implicit. Let $m$ be as in Lemma \ref{lem:gs}, and let $E
	= [m]$ and $F = \zo^n$. We take $\X = E \cup F$, equipped with the marginal
	distribution $\mD_\X$ so that $\Pr_{\mD_{\X}}[\x \in E] = 1 - \alpha$ and
	$\Pr_{\mD_{\X}}[\x \in F] = \alpha$ and the conditional distribution on each of $E$
	and $F$ is uniform.
    We will index the coordinates of $\rsh$ using $E$ and
	denote the $\x^{\textit{th}}$ coordinate of the encoding of $r$ by $\rsht$ as
	$\rsht(r)_\x$ for each $\x \in E$.

For each random seed $r \in \zo^n$, we define a Boolean function $g_r:\X \to \zo$ so that on the set $E$, $g_r$ is the encoding of $r$ using $\rsht$ whereas on $F$, $g_r$ equals the pseudorandom function $f_r$. Formally:
\begin{align*}
    g_r(\x) = \begin{cases} \rsht(r)_\x \ \text{for} \ \x \in E\\
    f_r(\x) \ \text{for} \ \x \in F
    \end{cases}
\end{align*}
We then define the $\pmo$ version of $g_r$ as $c_r(\x) = 2g_r(\x) -1$ and let $\mC = \{c_r ~|~ {r \in \zo^n}\}$. \smallskip

\noindent{\textbf{Part~\ref{thm:noise-hard-easy}}}: For the positive result in Part~\ref{thm:noise-hard-easy}, consider any distribution $\mD$ on $(\x, \y)$ where $\x \sim \mD_\X$ and there exists $c_r \in \mC$ such that $\cor(\y, c_r(\x)) \geq \alpha + \eps$. It must hold that $\cor(\y, c_r(\x)|\x \in E) \geq \eps/(1- \alpha)$, since if this were not true, we reach a contradiction as
\begin{align*} 
	\cor(\y, c_r)  &= (1 - \alpha) \cdot \cor(\y, c_r |\x \in E) + \alpha \cdot \cor(\y, c_r |\x \in F)\\
& < (1 - \alpha)\frac{\eps}{1 - \alpha} + \alpha\\
&= \alpha + \eps.
\end{align*}
Using the standard translation from correlation to Hamming distance for variables in $\pmo$, 
\begin{align*}
    \cor(\y, c_r(\x)|\x \in E) &= \E[c_r(\x)(2\y -1)|\x \in E]\\
    &= \E[(2g_r(\x) -1)(2\y -1)|\x \in E]\\
    &= 1 - 2\Pr[\y \neq g_r(\x)|\x \in E].
\end{align*}
Since the LHS is at least $\eps/(1- \alpha)$, and $g_r(x)$ equals $\rsht(r)_x$ on the set $E$, we conclude that
\begin{align*}
    \Pr_{(\x, \y) \sim \mD}[\y \neq \rsht(r)_\x|\x \in E] \leq \fr{2}\lt(1 - \frac{\eps}{1 - \alpha}\rt) \leq \frac{(1 - \eps - \eps\alpha)}{2}
\end{align*}
We wait until we see a sample $(\x = x, \y)$ for every $x \in [m]$, which takes $O(m\log(m)/(1- \alpha))$ samples, form a received word $\w$ and  decode it using the list-decoding algorithm. We can write

\begin{align*}
    \E[d(\w, \rsh(r))] &= \E\lt[\sum_{x \in [m]} \ind{\w(x) \neq \rsh(r)_x}\rt]\\
    &= \sum_{x \in [m]} \Pr[\w(x) \neq \rsh(r)_x]\\
    &= \sum_{x \in [m]} \Pr[\y \neq \rsh(r)_x|\x = x]\\
    &=  m\Pr_{(\x, \y) \sim \mD}[\y \neq \rsht(r)_\x|\x \in E]\\
    & \leq \frac{(1 - \eps - \eps\alpha)}{2}m.
\end{align*}
Hence by a Chernoff bound, with probability $1 -\exp(-\eps^2\alpha^2m)$  we have $d(\w, \rsh(r)) \leq (1- \eps)m/2$. So running the list-decoder from random samples returns a list containing the seed $r$. We can enumerate over all candidates $r'$ in the list, and check whether $\cor(\y,c_{r'}) \geq \alpha$ 
using random samples, using the fact that $f_r$ is easy to compute given $r$. Taking $r' = r$ will give $\alpha + \eps$ correlation, but there might be better hypotheses. We output the best hypothesis that we find, which will have correlation at least $\alpha$ with probability $0.9$.\footnote{We lose an $\eps$ in accounting for sampling error in estimating the true correlation.} We can amplify the probability of success through repetition. \smallskip

\noindent{\textbf{Part~\ref{thm:noise-hard-hardness1}}}: For the negative result in Part~\ref{thm:noise-hard-hardness1}, we define a distribution $\mD_r$  where we draw $r \in \zo^n$ at random.
Then $\y|\x$ is a uniformly random bit for $\x \in E$ and $\y|\x = f_r(\x)$ for $\x \in F$. In essence, $\y|\x$ is a truly random function on $E$ and a pseudorandom function on $F$. If we take the hypothesis $c_r$, then it achieves no correlation on $E$, and perfect correlation on $F$, so that $\cor(\y, c_r) = \alpha$. 
 
 On the other hand, consider an efficient learning algorithm which runs in time $T(n)$, is allowed query access to $\mD$, and outputs a hypothesis $h:\X \to [-1,1]$ that achieves correlation $\beta$. Then,
 \begin{align*}
    \beta &= (1 - \alpha)\cdot \cor(\y, h(\x)|\x \in E) + \alpha \cdot \cor(\y, h(\x)|\x \in F)\\
    &= \alpha \cdot \cor(f_r(\x), h(\x)|\x \in F).
 \end{align*}
 But this implies that $\cor(f_r(\x), h(\x)|\x \in F) \geq \beta/\alpha$, so $h$ (restricted to $F$) has distinguishing advantage $\beta/(2\alpha)$ for $f_r$ versus a random function, which is a contradiction if $\beta \geq 2\alpha/T(n)$. \smallskip

 \noindent{\textbf{Part~\ref{thm:noise-hard-hardness2}}}: Finally, for the negative result in Part~\ref{thm:noise-hard-hardness2}, consider the same distribution $\mD_r$ defined in Part~\ref{thm:noise-hard-hardness1} above. As observed above $\cor(\y, c_r) = \alpha$. Suppose $p$ is a $(\mC, \tau)$-multiaccurate predictor for $\mD_r$ for $\tau \leq \alpha/2 - \eps$. Let $\y_p$ denote the Bernoulli random variable corresponding to $p(\x)$, i.e., $\E[\y_p | p(\x)] = p(\x)$. Then we observe that since $p$ is ($\mC, \tau)$-multiaccurate, 
 \begin{align*}
     \lt|\cor(\y, c_r) - \cor(\y_p, c_r)\rt| & \leq \lt|\E[c_r(\x)(2\y -1) - c_r(\x)(2\y_p -1)]\rt|\\
     & \leq 2\lt|\E[c_r(\x)(\y - \y_p)]\rt|\\
     & \leq 2\lt|\E[c_r(\x)(\y - p(\x))]\rt|\\
     & \leq \alpha - 2\eps.
 \end{align*}
 Since $\cor(\y, c_r) \geq \alpha$, this implies that $\cor(\y_p, c_r) \geq 2\eps$. 

We will show that such a predictor $p$ can be used to predict $f_r$ with non-trivial probability, which contradicts Assumption \ref{claim:prf}. Since
     \begin{align*}
	 \cor(\y_p, c_r) &= (1 - \alpha) \cdot \cor(\y_p, c_r | \x \in E) + \alpha \cdot \cor(\y_p, c_r | \x \in F),
 \end{align*}
either $\cor(\y_p, c_r | \x \in E) \geq \eps/(1 - \alpha)$, or $\cor(\y_p, c_r  | \x \in F) \geq \eps/\alpha$. In the former case, we can use list decoding to obtain $r$ and thence $f_r$ as in Part~\ref{thm:noise-hard-easy}. For the latter case, we observe that $\y_p$ lets us predict $f_r$ with non-trivial advantage, since
 \begin{align*}
     \E_{\x \sim_u F}\lt[f_r(\x) \neq \y_p \rt] &= \fr{2}\E_{\x \sim_u F} \lt[1 -  (2\y_p  -1)(2f_r(\x) - 1)\rt]\\
     &= \fr{2} - \fr{2}\E_{\x \sim_u F}[(2\y_p - 1)c_r(\x)]\\
     &= \frac{1}{2} - \frac{1}{2} \cor(\y_p, c_r | \x \in F)
 \end{align*}
 where we use $c_r(x) = 2f_r(x) -1$ for $x \in F$. Thus, in either case we can predict the pseudorandom function $f_r$ with non-trivial advantage with non-negligible probability. This concludes the proof of the result.
\end{proof}

%% file: 8-conclusions.tex
\section{Conclusion and Future Work}

We have shown in this paper that multiaccuracy, together with global calibration, is a fairly powerful notion that has implications for agnostic learning and (optimal) hardcore set constructions. However, our work leaves open several open questions.

An immediate one is closing the gap in Theorem~\ref{thm:noise-hard} which only shows that an $(\alpha, \beta)$-weak agnostic learner does not yield anything better than $(\mC, \alpha/2)$-multiaccuracy; however, Theorem~\ref{claim:wal-ma-standard} only shows that it is possible to achive $(\mC, \alpha)$-multiaccuracy using access to an $(\alpha, \beta)$-weak agnostic learner.

Another natural question is whether more general reductions to \emph{multiaccuracy} can yield a weak agnostic learner. The result in Theorem~\ref{thm:ma-wal-} only rules out learners obtained by post-processing a predictor satisfying certain multigroup guarantees. A stronger model would be to consider an algorithm that has oracle access to a learner that for multiaccurate predictors, which it can call multiple times.  In Appendix~\ref{app:projection}, we show how a more \emph{general} post-processing involving points $x \in \X$ in addition to just the value $p(x)$ can in some cases yield agnostic learners.

Lastly, by the equivalence between the Regularity Lemma by \cite{TrevisanTV09} and the multiaccuracy theorem, multicalibration implies stronger and more general versions of other well-known theorems in complexity theory besides the hardcore lemma, such as the Dense Model Theorem and characterizations of pseudo-average min-entropy \cite{ReingoldTTV08, vadhan2012characterizing, casacuberta2024complexity}, and Yao's XOR lemma \cite{marcussen2024characterizing}.
Similar to our result for Impagliazzo's Hardcore Lemma,
it is natural to ask if versions of these theorems can be proved under weaker assumptions like calibrated multiaccuracy.

\section*{Acknowledgments} 

We thank Igor Carboni Oliveira, Vitaly Feldman, Venkatesan Guruswami, Rahul Santhanam, and Salil Vadhan for helpful discussions and feedback. 

SC is supported by a Rhodes Scholarship.
OR is supported by the Simons Foundation Collaboration on the Theory of Algorithmic Fairness, the Sloan Foundation Grant 2020-13941 and the Simons Foundation investigators award 689988.

%% file: 9-projection.tex
\section{Projecting Multiaccurate Predictors onto the Span of $\mC$} \label{app:projection}

Intuitively, the reason why $\C$-multiaccuracy does not directly imply weak
agnostic learning for the class $\C$ is that multiaccuracy only implies a
certain type of closeness between the predictor $p$ and the labels $\y$ over
the span of the concept class $\C$. Let $p^*$ denote the Bayes optimal
predictor, i.e., $p^*(\x) = \E[\y | \x]$, and suppose for now that $p$
satisfies $(\C, 0)$-multiaccuracy for $\mD$. Let $\tp^* = 2 p^* - 1$ and $\tp =
2p - 1$ denote the $[-1, 1]$ versions of the predictors $p^*$ and $p$.  Let
$\lspan(\C)$ denote the vector space spanned by $\C$. We can write,
\[ \tp = \tp_C + \tp_{\perp} \text{~~and~~} \tp^* = \tp^*_C + \tp^*_\perp, \] 
where $\tp_C, \tp^*_C$ are the components of $\tp$ and $\tp^*$ in the $\lspan(\mC)$
respectively, and $\tp_\perp, \tp^*_\perp$ are the components of $\tp, \tp^*$
orthogonal to $\lspan(\mC)$. We recall that the inner product $\la f, g \ra$ for
functions $f, g : \X \to \R$ is defined by $\E_{\mD_{\X}}[f(\x) g(\x)]$. Since $p$ is
$(\C, 0)$-multiaccurate for $\mD$, we have that, 
\[ \la c, \tp - \tp^* \ra = 2 \E[c(\x) (p(\x) - p^*(\x))] = 0 \] 
for every $c \in C$; i.e., $\tp- \tp^*$ is orthogonal to $\lspan(\mC)$. But 
\[ \tp - \tp^* = \tp_C - \tp_C^* + \tp_\perp - \tp_\perp^*. \] 
As $\tp_\perp, \tp^*_\perp$ are both orthogonal to $\lspan(\mC)$, and $\tp_C - \tp^*_C
\in \lspan(\mC)$, we must have $\tp_C = \tp^*_C$. What this suggests is that even
though technically we can have 
\[ \cor(\y, 2p - 1) = \la \tp^*, \tp \ra = \E_{\x \sim \mD_\X}[(2p^*(\x) - 1)(2p(x) - 1)] = 0  \]
as Theorem~\ref{thm:ma-wal-} indicates, we can conclude that,
\begin{align*}
\cor(\y, \tp_C) = \la \tp^*, \tp_C \ra &= \la \tp^*_C + \tp^*_\perp, \tp_C \ra \\
													  &= \la \tp^*_C, \tp_C \ra + \la \tp^*_\perp, \tp_C \ra  \\
													  &= \Vert \tp^*_C \Vert_2^2.
\end{align*}
Above in the last line we used that $\tp^*_\perp$ is orthogonal to
$\lspan(\mC)$ and that $\tp_C \in \lspan(\mC)$, as well as the observation that
$\tp_C = \tp^*_C$. Thus, a projection of $\tp$ onto $\lspan(\mC)$ could
potentially lead to a weak agnostic learner. However, we need to take some care
as (a) $\tp_C$ need not have range $[-1, 1]$ (or be bounded even), and (b) if
we don't have $(\mC, 0)$-multiaccuracy, but only $(\mC, \tau)$, then depending
on the representation of $\tp_C$ in terms of $c \in C$, the errors could
accumulate making the above analysis trivial. Observe that this constitutes a
much richer form of post-processing that considered in
Section~\ref{sec:WAL-MA}.

In this section, we make this intuition more precise. 

\begin{theorem} \label{MA+proj-gives-WAL}
	Let $p: \X \rightarrow [0,1]$ satisfy $(\C, \tau)$-multiaccuracy for $\mD$. Suppose we have $q = \sum_{s} \lambda c_s$ for $c_s \in \C$, satisfying,
	\[ \E_{\x \sim \mD_\X}[ ((2p(\x)-1)  - q(\x))^2] \leq \E_{\x \sim \mD_\X}[ (2p(\x) - 1)^2] - \gamma. \]
    Let $h = \clip(q)$, which is clipped to have $h(\x) \in [-1, 1]$ for all $\x \in \X$. Then,
	 \[ \E_{(\x, \y) \sim \mD}[ ((2\y - 1) - h(\x))^2] \leq 1 - \gamma + 4 \tau \sum_{s} |\lambda_s|. \]
\end{theorem}
\begin{proof}
	For simplicity, let $\tilde{p} = 2p - 1$ denote the $[-1, 1]$-version of $p$.  Then,
	\begin{align*}
		\gamma &\leq \E_{\x \sim \mD_\X}[\tilde{p}(\x)^2] - \E_{\x \sim \mD_\X}[(\tilde{p}(\x) - q(\x))^2]  \\
				 &= 2 \E_{\x \sim \mD_\X}[ \tilde{p}(\x) q(\x)] - \E_{\x \sim \mD_\X}[q(\x)^2] \\ 
				 &= 2 \E_{(\x, \y) \sim \mD}[ (2\y -1) q(\x)] - \E_{\x \sim \mD_\X}[q(\x)^2] + 2 \E_{(\x, \y) \sim \mD}[q(\x)(\tilde{p}(\x) - (2\y-1))] \\ 
				 &= 2 \E_{(\x, \y) \sim \mD}[ (2 \y - 1) q(\x)] - \E_{\x \sim \mD_\X}[q(\x)^2] + 4 \E_{(\x, \y) \sim \mD}[q(\x)(p(\x) - y(\x))] \\ 
				 &= 2 \E_{(\x, \y) \sim \mD}[ (2\y -1) q(\x)]  - \E_{\x \sim \mD_\X}[q(\x)^2] + 4 \sum_{s} \lambda_s \E_{(\x, \y) \sim \mD}[c_s(\x)(p(\x) - \y)] \\ 
				 &\leq \E_{(\x, \y) \sim \mD}[(2 \y - 1)^2] - \E_{(\x, \y) \sim \mD}[((2 \y -1)  - q(\x))^2] + 4 \tau \sum_{s} | \lambda_s| \\ 
				 &\leq 1 - \E_{(\x, \y) \sim \mD}[((2 \y -1) - h(\x))^2] + 4 \tau \sum_{s} |\lambda_s|.
	\end{align*}
	Above in the second last inequality, we used the fact that $p$ is $(\C, \tau)$-multiaccurate for $\mD$ and that $c_s \in \C$ for all $s$, and in the last inequality we use that \emph{clipping} $q$ to get $h \in [-1,1]$ only reduces the squared error with respect to $2\y-1$ and that $2\y - 1 \in \{-1, 1\}$. The conclusion then follows by rearranging.
\end{proof}


Together, Theorem~\ref{MA+proj-gives-WAL} and Theorem~\ref{thm:h-projection} imply that a certain type of \emph{projection} onto $\lspan(\mC)$, particularly one that is $\ell_1$-sparse, suffices for weak agnostic learning:

\begin{corollary}
    Let $p: \X \rightarrow [0,1]$ satisfy $(\C, \tau)$-multiaccuracy for $\mD$. Suppose we have $q = \sum_{s} \lambda_s c_s$ for $c_s \in \C$, satisfying,
	 \[ \E_{\x \sim \mD_\X}[ ((2p(\x)-1)  - q(\x))^2] \leq \E_{\x \sim \mD_\X}[ (2p(\x) - 1)^2] - \gamma. \]
	for $\gamma \geq 4 \tau \sum_{s} |\lambda_s| + \beta$.
	Then, $h = \clip(q)$, which clips $q$ to the range $[-1, 1]$, satisfies $\cor_\mD(\y, h) \geq \beta$. 
\end{corollary}

Conversely, it follows immediately that if we have a \emph{proper} weak agnostic learner that given \emph{black box} access to $p$ returns $c \in C$, such that $\cor_\mD(p, c) \geq \alpha > 0$, then $q = \alpha c/\Vert c \Vert$, satisfies,
\begin{align*}
	\E_{\x \sim \mD_\X}[((2p(\x) - 1) - q(\x))^2] &= \E_{\x \sim \mD_\X}[(2 p(\x) - 1)^2] - 2 \E_{\x \sim \mD_\X} [ (2p(\x) - 1) q(\x)] + \E_{\x \sim \mD_\X}[q(\x)^2] \\
																 &= \E_{\x \sim \mD_\X}[(2 p(\x) - 1)^2] - 2 \frac{\alpha}{\Vert c \Vert} \cor(p, c) + \alpha^2 \\
																 &\leq \E_{\x \sim \mD_\X}[(2 p(\x) - 1)^2] -  \alpha^2. \\
\end{align*}
In the last inequality, we used $\cor(p, c) \geq \alpha$ and $\Vert c \Vert \leq 1$, as $c$ has range $[-1, 1]$. 

\subsection{Multiaccuracy as a Query Oracle}

A natural question is whether there are any cases where such $\ell_1$-sparse projections are possible. One advantage of having access to the predictor $p$ rather than just examples $(\x, \y)$ drawn from $\mD$ is that we may query the predictor $p$ at any point $\x \in \X$ of our choice. As a result, if we have a \emph{proper} learning algorithm for a class $\mC$ that uses membership queries, then we can learn such a class only with random examples and access to a multiaccurate predictor $p$. 

\begin{theorem} \label{thm:MA-implies-MQ-to-PAC} Suppose the class $\C$ (that contains the constant function $1$) is efficiently \emph{properly} agnostically learnable with access to random examples from $\mD$ and query access to the target function to accuracy $\eps$, then provided we have access to a $(C, \tau)$-multiaccurate predictor $p$ for $\mD$, $\C$ is efficiently \emph{properly} agnostically learnable using only access to random examples from $\mD$ and (using the predictor $p$) to accuracy $\eps + 3 \tau$.
\end{theorem}
\begin{proof}
	In this proof, we will assume that the target function $f$ takes values in $\zo$ and also that $\mC$ is a class of functions from $\X \to \zo$. Let $\mA$ be the algorithm for efficiently agnostically learning $\C$ using the random examples from $\mD$ and query access to the target function and $p$ be the $(\mC, \tau)$-multiaccurate predictor for $\mD$. Since $p$ may not be Boolean, we denote by $\mathbf{Z}_{p(\x)}$ a (freshly independently drawn each time) Bernoulli random variable with parameter $p(\x)$. When the algorithm draws a random example, we replace $(\x, \y)$, by $(\x, \mathbf{Z}_{p(\x)})$ and we simulate queries by returning $\mathbf{Z}_{p(\x)}$ when queried with point $\x$. Then, we are guaranteed that $\mA$ returns $c_\mA \in \C$, such that with probability at least $0.99$, 
	\begin{align*}
       \err(c_\mA; \mD) &\leq \min_{c \in \mC} \E[|c(\x) - p(\x)|] + \eps
   \end{align*}

	Let $f : \X \rightarrow \{0, 1\}$ denote the target, i.e. we get $(\x, f(\x))$ from $\mD$. Then, we have the following,
   \begin{align*}
       \Pr[c(\x) \neq f(\x)] &= \E[c(\x) (1 - f(\x)) + (1 - c(\x)) f(\x)] \\
       &= \E[c(\x) + f(\x)] - 2 \E[c(\x) f(\x)] \\
       &\leq \E[c(\x) + p(\x)] - 2 \E[c(\x) p(\x)] + 3 \tau \\
       &= \E[c(\x)( 1 - p(\x)) + (1 - c(\x)) p(\x)]  + 3 \tau \\
       &= \E[|c(\x) - p(\x)|] + 3 \tau.
       \intertext{Similarly, we may show that}
       \Pr[c(\x) \neq f(\x)] &\geq \E[|c(\x) - p(\x)|] - 3 \tau.
   \end{align*}
   Above we have used that $p$ is $(\mC, \tau)$-multiaccurate for $\mD$ as well as assumed that the constant $\mathbf{1}$ function is in $\mC$. 

   Thus, we get that,
   \begin{align*}
   \err(c_\mA; \mD) &\leq \min_{c \in C} \err(c; \mD) + \eps + 3 \delta.
   \end{align*}
\end{proof}
When considering learning functions defined over $\{-1, 1\}^n$ with respect to the uniform (or in some cases a product) distribution over the domain, a classical result of Goldreich and Levin shows that with query access to the target function, we can identify all the \emph{heavy} Fourier coefficients. Essentially this algorithm allows us to find projections as required by Theorem~\ref{MA+proj-gives-WAL}; this also gives a proper agnostic learning algorithm for the class of parities. We briefly the Goldreich-Levin algorithm below.

\paragraph{The Goldreich-Levin algorithm.}

By the Fourier expansion theorem we know that every function $f: \{-1, 1\} \rightarrow \mathbb{R}$ can be uniquely expressed as a multilinear polynomial:
\[
    f(x) = \sum_{S \subseteq [n]} \hat{f}(S) \chi_S,
\]
where $\chi_S(x) = \prod_{i \in S} x_i$ for $x = (x_1, \ldots, x_n) \in \mathbb{R}^n$ and $\hat{f}(S) = \E_{x \sim \{-1, 1\}^n}[f(x) \chi_S]$
denotes the Fourier coefficient of $f$ on $S$. 
Again defining the inner product $\langle f, g \rangle$ between two functions $f, g: \{-1, 1\}^n \rightarrow \mathbb{R}$ as $\E_{x \sim \{-1, 1\}^n}[f(x) g(x)]$ we see that the $2^n$ parity functions $\chi_S: \{-1, 1\}^n \rightarrow \{-1, 1\}$ form an orthonormal basis for the vector space of functions $\{-1, 1\}^n \rightarrow \mathbb{R}$ \cite[Thm. 1.5]{o2014analysis}. 
The Goldreich-Levin algorithm finds the largest Fourier coefficients of a function given membership access to it:

\begin{theorem}[Goldreich-Levin \cite{goldreich1989hard}]\label{thm:gl}
Given query access to \( f : \{-1, 1\}^n \to [-1, 1] \), and given \( \gamma, \delta > 0 \), 
there is a \( \text{poly}(n, \frac{1}{\gamma}, \log \frac{1}{\delta}) \)-time algorithm that outputs a list 
\( L = \{S_1, \dots, S_m\} \) such that:
\begin{enumerate}
    \item If \( \hat{f}(S) \geq \gamma \), then \( S \in L \), and
    \item If \( S \in L \), then \( \hat{f}(S) \geq \frac{\gamma}{2} \) holds with probability \( 1 - \delta \).
\end{enumerate}
\end{theorem}

Using Theorem~\ref{thm:MA-implies-MQ-to-PAC}, we can observe the following
simple corollary which shows that finding $p$ that is $(\mC, \tau)$-accurate
for $\mD$, where the marginal distribution over $\X$ is uniform, is at least as
hard as Learning Parities with Noise (LPN).  Since LPN is widely believed to be
hard, we can conclude that the existence of an efficient algorithm that
produces a $p$ that is $(\mC, \tau)$-multiaccurate for small values of $\tau$
is unlikely to exist, at least when the marginal distribution over $\mX$ is
uniform.

\begin{corollary}
    When $\mC$ is the class of parities, and $\mD$ is such that the marginal distribution over $\mX$ is uniform, then finding $p$ that is $(\mC, \tau)$-multiaccurate for $\mD$ and $\tau \leq 0.1$ is at least as hard as \emph{properly agnostically} learning $\mC$ with random classification noise with noise rate $0.1$. 
\end{corollary}
\begin{proof}
	We know that the class of parities is \emph{efficiently properly} agnostically learnable using the Goldreich-Levin or the Kushilevitz-Mansour algorithm \cite{goldreich1989hard, kushilevitz_learning_1993}. Since the noise rate is $0.1$, we know that there the optimal error is at most $0.1$, then using Theorem~\ref{thm:MA-implies-MQ-to-PAC} we can find a parity that has error at most $0.4 + \eps < 1/2$. However, in the random classification noise setting all parities, except the target parity are uncorrelated with the target labels, so in fact we must have identified the target parity.
\end{proof}